\documentclass{article}

\usepackage[preprint]{neurips_2026}

\usepackage[utf8]{inputenc}
\usepackage[T1]{fontenc}
\usepackage{microtype}
\usepackage{graphicx}
\usepackage{subfigure}
\usepackage{booktabs}
\usepackage{multirow}
\usepackage{makecell}
\usepackage{hyperref}
\usepackage{url}
\usepackage[defaultlines=2,all]{nowidow}

\usepackage{algorithm}
\usepackage{algorithmic}

\usepackage{amsmath,amssymb,amsfonts}
\usepackage{mathtools}
\usepackage{amsthm}
\usepackage{siunitx}
\usepackage{nicefrac}
\usepackage{xcolor}

\usepackage{amsmath,amsfonts,bm}

\def\Secref#1{Section~\ref{#1}}

\def\eqref#1{equation~\ref{#1}}

\def\1{\bm{1}}

\DeclareMathAlphabet{\mathsfit}{\encodingdefault}{\sfdefault}{m}{sl}
\SetMathAlphabet{\mathsfit}{bold}{\encodingdefault}{\sfdefault}{bx}{n}

\newcommand{\E}{\mathbb{E}}

\newcommand{\KL}{D_{\mathrm{KL}}}

\newcommand{\Cov}{\mathrm{Cov}}

\theoremstyle{plain}
\newtheorem{assumption}{Assumption}
\newtheorem{lemma}{Lemma}
\newtheorem{proposition}{Proposition}
\newtheorem{theorem}{Theorem}
\newtheorem{definition}[theorem]{Definition}
\newtheorem{corollary}[theorem]{Corollary}

\usepackage[capitalize,noabbrev]{cleveref}
\crefname{assumption}{Assumption}{Assumptions}
\Crefname{assumption}{Assumption}{Assumptions}
\crefname{lemma}{Lemma}{Lemmas}
\Crefname{lemma}{Lemma}{Lemmas}
\crefname{proposition}{Proposition}{Propositions}
\Crefname{proposition}{Proposition}{Propositions}
\crefname{theorem}{Theorem}{Theorems}
\Crefname{theorem}{Theorem}{Theorems}
\crefname{definition}{Definition}{Definitions}
\Crefname{definition}{Definition}{Definitions}
\crefname{corollary}{Corollary}{Corollaries}
\Crefname{corollary}{Corollary}{Corollaries}
\usepackage{wrapfig}
\usepackage{float}
\usepackage{enumitem}

\newcommand{\THL}{THL}
\newcommand{\ThlFull}{Tokenizer Heterogeneity Layer}
\newcommand{\ThlFullFirstAppear}{\underline{T}okenizer \underline{H}eterogeneity \underline{L}ayer}
\newcommand{\SEE}{SEE}
\newcommand{\SeeFull}{Shared Experience Exchange}
\newcommand{\SeeFullFirstAppear}{\underline{S}hared \underline{E}xperience \underline{E}xchange}
\newcommand{\MWRA}{MWRA}
\newcommand{\MwraFull}{Multi-Worker Resource Allocation}
\newcommand{\MwraFullFirstAppear}{\underline{M}ulti-\underline{W}orker \underline{R}esource \underline{A}llocation}

\newcommand{\MOneShort}{PRP}

\newcommand{\MOneFullFirstAppear}{\underline{P}eer \underline{R}ollout \underline{P}ooling}

\newcommand{\MThreeShort}{XGRPO}

\newcommand{\MThreeFullFirstAppear}{\underline{C}ross-Policy \underline{G}RPO \underline{A}dvantage \underline{S}haring}
\newcommand{\MFourShort}{SGT}

\newcommand{\MFourFullFirstAppear}{\underline{S}uccess-\underline{G}ated \underline{T}ransfer}

\title{Experience Sharing in Mutual Reinforcement Learning for Heterogeneous Language Models}

\author{%
  Xiaoze Liu$^{1,2}$\thanks{Work done during an internship at Amazon.}, Dhananjay Ram$^2$, Yuting Zhang$^2$ \\
  \textbf{Zhaoyang Zhang$^2$, Wei Xia$^2$, Stefano Soatto$^2$} \\
  $^1$Purdue University \qquad $^2$AWS Agentic AI
}

\begin{document}

\maketitle

\begin{abstract}
We introduce \emph{Mutual Reinforcement Learning}, a framework for concurrent RL post-training in which heterogeneous LLM policies exchange typed experience while keeping separate parameters, objectives, and tokenizers. The framework combines a \emph{Shared Experience Exchange} (SEE), \emph{Multi-Worker Resource Allocation} (MWRA), and a \emph{Tokenizer Heterogeneity Layer} (THL) that retokenizes text and aligns token-level traces across incompatible vocabularies. This substrate makes the experience-sharing design question operational across model families. We instantiate three controlled probes on top of GRPO: data-level rollout sharing via \emph{Peer Rollout Pooling} (PRP), value-level advantage sharing via \emph{Cross-Policy GRPO Advantage Sharing} (XGRPO), and outcome-level success transfer via \emph{Success-Gated Transfer} (SGT). A contextual-bandit analysis characterizes their structural positions on a stability-support trade-off: PRP pays density-ratio variance and THL residual costs, XGRPO preserves on-policy actor support while changing scalar baselines, and SGT supplies a rescue-set score direction toward verified peer successes. In the evaluated regime, outcome-level sharing occupies the favorable point of this trade-off.
\end{abstract}

\section{Introduction}
\label{sec:introduction}

Reinforcement learning (RL) for large language models (LLMs) has largely remained a single-policy endeavor, typically relying on PPO-style surrogates with KL regularization~\citep{Schulman2017PPO,Ouyang2022InstructGPT,Ziegler2019RLHF}.
The field has shifted from critic-based estimation (e.g., GAE~\citep{Schulman2016GAE}) toward verifiable, critic-free group-relative variants~\citep{shao2024deepseekmath,guo2025deepseek,zheng2025group,yu2025dapo,liu2026gdpo}, but the dominant ``one model at a time'' training loop persists.
In this siloed workflow, exploration and improvement are coupled to a single policy.
Complementary experiences discovered by concurrently trained peers are discarded~\citep{Wortsman2022ModelSoups,Wang2023SelfConsistency}, leaving no principled channel for policies to learn from one another \emph{during} training.

This isolation is inefficient in sparse-reward reasoning domains.
Heterogeneous policies (varying in architecture or scale) exhibit diverse inductive biases and traverse distinct regions of the solution space~\citep{fort2019deep,allen2020towards,wang2024mixture}.
When one policy solves a problem that remains unsolved by its peers, that trajectory is a high-value experience currently used only by the successful policy. Mutual RL turns such isolated discoveries into typed training signals for the other policies in the pool.
This amplifies sparse verifier feedback: a policy can receive evidence of a correct reasoning path on prompts where its own rollout group produced none.
We introduce \emph{mutual reinforcement learning} as concurrent post-training in which each policy learns from its own on-policy rollouts and selected peer signals while keeping separate parameters, objectives, and tokenizers.
A peer-discovered behavior can therefore accelerate the acquisition of
complex capabilities like reasoning~\citep{Wei2022CoT,Zelikman2022STAR}.
Rather than replacing standard pipelines with complex off-policy algorithms~\cite{wu2024selfplay,kostrikov2021offline,haarnoja2018soft}, we make peer experiences usable \emph{inside} the on-policy surrogates used by contemporary LLM RL~\cite{zheng2025group, shao2024deepseekmath,guo2025deepseek}.

Realizing this vision requires addressing two design questions. First, on the framework side, policies often use disparate tokenizers, so peer experiences must be translated into a representation each learner can consume~\citep{rust2021good,dobler-de-melo-2023-focus}. Second, at the algorithmic level, directly reusing peer rollouts inside on-policy surrogates introduces distribution shift~\citep{degris2012offpolicy,Espeholt2018IMPALA,liu2020offpolicy,zhang2019geoffpac,kumar2020cql,zheng2025prosperitycollapsefaroffpolicy,fu2025areal,noukhovitch2024asynchronous}. Until heterogeneous policies can exchange comparable signals (token-level traces, scalar rewards, and verified successful trajectories) at all, the question of \emph{which experience-sharing regime is the right abstraction to adopt} is not even well-posed. The framework in Sec.~\ref{sec:system-design} is the prerequisite that makes this question askable. With it in place, we then ask: \emph{within the now-standard regime of single-step verifiable post-training for open-weights LLMs, which experience-sharing regime is appropriate, and what trade-offs do alternative regimes induce?}

We first build the substrate that lets heterogeneous models communicate during RL training. We extend VERL~\citep{sheng2024hybridflow}, one of the most widely used RL frameworks for LLMs, with a typed dataflow in which distinct policies exchange structured experience (rollouts, rewards, advantages, and optional traces) while preserving native per-policy objectives. \MwraFullFirstAppear{} assigns workers and budgets across policies, and \SeeFullFirstAppear{} stores typed records with provenance. To bridge incompatible vocabularies, we
develop a \ThlFullFirstAppear{} that retokenizes peer rollouts and aligns token-level traces across model families.
With the framework in place, we populate the stability-support design space with three controlled probes on top of the same GRPO base update~\citep{shao2024deepseekmath}. Each probe exposes a different level of experience sharing and can be compared under matched compute and analyzed structurally in App.~\ref{app:theory}.
\begin{itemize}
[topsep=0pt,itemsep=0pt,parsep=0pt,partopsep=0pt,leftmargin=*]
    \item \textbf{Regime 1: Data-level experience sharing.} \MOneFullFirstAppear{} treats the shared pool as supplementary training trajectories and tests the density-ratio cost of direct peer-rollout reuse.
    \item \textbf{Regime 2: Value-level experience sharing.} \MThreeFullFirstAppear{} normalizes the learner's critic-free advantages with cross-policy reward statistics while keeping the actor strictly on-policy.
    \item \textbf{Regime 3: Outcome-level experience sharing.} \MFourFullFirstAppear{} treats peer outputs as sparse outcome certificates rather than general off-policy rollouts. The gate activates only when the learner fails on a prompt that a peer has solved; on this rescue subset, the verified peer trajectory is injected as a sparse positive sample alongside the learner's own on-policy GRPO update. \MFourShort{} is therefore peer-conditioned, failure-triggered, and interleaved with the base RL objective.
\end{itemize}

To summarize, our contributions are as follows:

\begin{enumerate}[topsep=0pt,itemsep=0pt,parsep=0pt,partopsep=0pt,leftmargin=*]
    \item \textbf{Framework.} We introduce a reproducible framework for mutual RL across heterogeneous LLM policies, comprising \SeeFull{} (\SEE{}) for typed cross-policy experience, \MwraFull{} (\MWRA{}) for per-policy budgets, and \ThlFull{} (\THL{}) for text and token-level signal exchange across incompatible vocabularies. \THL{} provides the cross-tokenizer interface for text and token-level trace exchange, with bounded alignment residual analyzed in App.~\ref{app:theory}.
    \item \textbf{Three controlled experience-sharing probes.} We instantiate \MOneShort{} (data-level), \MThreeShort{} (value-level), and \MFourShort{} (outcome-level) on the same GRPO substrate and compare them under matched compute, hyperparameters, and deterministic validation decoding within a single VERL-based stack. The probes populate the stability-support design space at three distinct coupling levels.
    \item \textbf{Theoretical characterization of the stability-support trade-off.} A contextual-bandit theoretical analysis (App.~\ref{app:theory}) proves the structural differences between the three probes: density-ratio variance and THL-residual sensitivity for \MOneShort{}, on-policy support preservation for \MThreeShort{}, and a rescue-set positive-gradient direction for \MFourShort{}. This grounds the design space in update structure rather than per-task benchmark variation.
\end{enumerate}

\section{Related Work} 
\label{sec:related}

\paragraph{RL for LLMs.}
RL for post-training LLMs typically adopts PPO-style surrogates with KL regularization and either critic-based (e.g., GAE) or critic-free (e.g., group-relative) advantages, often coupled with outcome- or process-level reward models and verifiers~\citep{Schulman2017PPO,Schulman2016GAE,Williams1992REINFORCE,Schulman2015TRPO,Christiano2017Preferences,Stiennon2020Summarize,Ouyang2022InstructGPT,Ziegler2019RLHF,Bai2022ConstitutionalAI,Rafailov2023DPO,Cobbe2021TrainingVerifiers,Lewkowycz2022Minerva,Wei2022CoT,Wang2023SelfConsistency,Yao2023ReAct,Yao2023ToT}. On math-focused reasoning tasks, GRPO-style pipelines and their variants further refine these templates through verifiable, test-time, and data-efficient RL on language models, as well as intrinsic and noisy reward signals~\citep{shao2024deepseekmath,guo2025deepseek,TTRL,1-shot-RL,Few-Shot-RL,Absolute-Zero,one-shot-EM,RENT,EM-RL,INTUITOR,Self-Train,lv-PRP,NoFreeLunch,shao2025spurious}. A growing body of work also revisits the underlying algorithms, proposing value-augmented and trust-region variants to improve stability and credit assignment in RLVR-style regimes~\citep{kazemnejad2024vineppo,yuan2025vc_ppo,yuan2025vapo,yu2025dapo,hu2025reinforce_pp,zhang2025srpo}. In parallel, off-policy RL theory highlights how naively combining replay with policy gradients can be unstable, motivating guards such as clipped ratios and KL regularization~\citep{degris2012offpolicy,Espeholt2018IMPALA,liu2020offpolicy,zhang2019geoffpac,kumar2020cql,Mnih2016A3C,Levine2018RLasInference,Schulman2017PPO,Schulman2016GAE}, and sequence-level RL for text predates modern RLHF with similar surrogates for translation and summarization~\citep{Ranzato2016MIXER,Bahdanau2017ActorCritic,Norouzi2016RAML,Rennie2017SCST,Paulus2018SummarizationRL}. Complementary system work explores asynchronous and large-scale RLHF frameworks that decouple rollouts from learning while carefully managing staleness~\citep{fu2025areal,noukhovitch2024asynchronous,zhong2025streamrl,he2025history}.
Our work keeps the GRPO/PPO-style training interface but changes the dataflow: concurrently trained policies can publish and consume typed peer fields. This lets us compare which level of sharing is stable inside the standard on-policy surrogate family rather than replacing that family with a separate off-policy algorithm.

\paragraph{LLM Ensembling and Merging.}
A large body of work combines LLMs at \emph{inference} time via ensembling, verifier or retriever-based reranking, and multi-agent debate or critique to improve accuracy and calibration without changing weights~\citep{Wang2023SelfConsistency,Wei2022CoT,Yao2023ReAct,Yao2023ToT,Shinn2023Reflexion,Lewis2020RAG,Izacard2021FiD,Irving2018Debate,Zhou2023LeastToMost,Asai2023SelfRAG,Zelikman2022STAR,Gao2023PAL}. Orthogonal work merges parameters through weight averaging such as model soups or task arithmetic, interpolating capabilities at the cost of potential interference and tokenizer constraints~\citep{Wortsman2022ModelSoups,Matena2022FisherMerging,Ilharco2023TaskVectors,Hu2022LoRA,Pfeiffer2021AdapterFusion,BenZaken2022BitFit,Lester2021PromptTuning,Li2021PrefixTuning}.
Recent work also explores concurrent training of multiple LLMs in multi-agent RL settings, using distributed infrastructure to coordinate heterogeneous roles and collaborative problem solving~\citep{zhang2025marti,feng2025stronger,zhao2025maporl,li2025interactive}.
Mutual RL differs along three axes. It keeps model parameters separate rather than merging weights, changes the trained policy rather than only ensembling at inference, and improves individual policies rather than optimizing a specialized multi-agent team. The shared object is experience, not parameters or test-time votes.

\paragraph{Mutual Learning and Multi-Student Distillation.}
Traditional mutual learning trains multiple students jointly via logits matching, co-distillation, or co-teaching strategies that filter noisy labels; related ideas appear in co-training and multi-agent RL, which use agreement or shared critics to stabilize decentralized policies~\citep{Zhang2018DML,Furlanello2018BAN,Hinton2015Distillation,Romero2015FitNets,Han2018CoTeaching,Jiang2018MentorNet,Blum1998CoTraining,Lowe2017MADDPG,Rashid2018QMIX,Foerster2018COMA,Sunehag2018VDN,Tarvainen2017MeanTeacher,Xie2020NoisyStudent,Kim2016SeqKD,Jiao2020TinyBERT,Sanh2019DistilBERT,Laine2017Temporal,FixMatch2020,Son2019QTRAN,Iqbal2019MAAC}.
Our setting exchanges signals native to LLM RL, including rollouts, scalar rewards, group-normalized advantages, and verified successful responses. \MFourShort{} is not online distillation: it is gated on learner failure, anchored by the concurrent on-policy GRPO update, and transfers outcome-verified trajectories rather than matching peer logits or full output distributions.

\section{Preliminaries}
\label{sec:preliminaries}
\paragraph{Group-Relative Policy Optimization (GRPO).}
Let $x$ be a prompt and $y=(y_1,\dots,y_T)$ a sampled response.
We optimize a current policy $\pi_\theta$ using samples generated by a fixed behavior policy
$\pi_{\theta_{\mathrm{old}}}$.
For each $(x,y)$, we write
$
\ell_\theta(x,y)\in\mathbb{R}^{|y|}
\quad\text{and}\quad
\ell_{\mathrm{old}}(x,y)\in\mathbb{R}^{|y|}
$
for the vectors of token-level log-probabilities under $\pi_\theta$ and
$\pi_{\theta_{\mathrm{old}}}$, respectively; their entries are
$\log \pi_\theta(y_t\mid x,y_{<t})$ and
$\log \pi_{\theta_{\mathrm{old}}}(y_t\mid x,y_{<t})$ for each token position $t$.
The token-level importance weights are defined elementwise by
$
w_\theta(x,y) = \exp\big(\ell_\theta(x,y)-\ell_{\mathrm{old}}(x,y)\big)\in\mathbb{R}^{|y|}.
$
For each prompt $x$ we draw a group of $K$ responses
$\{y_i\}_{i=1}^{K}\sim \pi_{\theta_{\mathrm{old}}}(\cdot\mid x)$
with scalar rewards $\{r_i\}_{i=1}^{K}$, and form group-relative advantages
\[
\bar r = \tfrac{1}{K}\sum_{j=1}^{K} r_j,
\qquad
\tilde A_i = r_i-\bar r,
\qquad
\tilde A_i^{\mathrm{zn}} = \frac{\tilde A_i}{\mathrm{std}(\{r_j\})+\epsilon}.
\]
Let $A_i$ denote the sequence-level advantage for $y_i$, i.e.,
$A_i = \tilde A_i$ (or $\tilde A_i^{\mathrm{zn}}$).
We denote by $\mathbb{E}_{\mathrm{tok}}[\cdot]$ the empirical average over token positions in $y_i$, and by $w_{\theta,t}(x,y_i)$ the per-token importance ratio at position $t$. With clipping threshold $\epsilon$, the GRPO loss we minimize is the standard PPO-style surrogate aggregated over the rollout group:
\begin{equation}
\label{eq:grpo}
\mathcal{L}_{\mathrm{GRPO}}(\theta)
=
-\,\mathbb{E}_{x}\,
\frac{1}{K}\sum_{i=1}^{K}
\mathbb{E}_{\mathrm{tok}}\!\Big[
\min\!\big(w_{\theta,t}(x,y_i)\,A_i,\;\mathrm{clip}\!\big(w_{\theta,t}(x,y_i),1-\epsilon,1+\epsilon\big)\,A_i\big)
\Big]
+\beta\,\mathcal{R}(\theta).
\end{equation}

where $\mathcal{R}(\theta)$ collects KL-style regularization to a reference model.
The behavior policy $\pi_{\theta_{\mathrm{old}}}$ is treated as fixed (we do not differentiate through it).
\paragraph{Mutual RL Setup.}
We train $M$ policies $\{\pi^{(m)}\}_{m=1}^{M}$ concurrently on a shared dataset. At each training step, every policy receives the \emph{identical} batch of prompts $x$.
Policy $m$ generates on-policy rollouts $y^{(m)} \sim \pi^{(m)}(\cdot\mid x)$ and publishes selected fields from those rollouts to the shared exchange for subscribed peers.
We denote by $\ell^{(a)}(x, y^{(b)})$ the token-level log-probabilities of policy $\pi^{(a)}$ evaluating a response $y^{(b)}$ originally generated by policy $\pi^{(b)}$.
This notation lets the three probes distinguish which cross-policy field is consumed: complete trajectories in \MOneShort{}, scalar reward statistics in \MThreeShort{}, or verified successful responses in \MFourShort{}.

\begin{figure*}[t]
  \centering
  \includegraphics[width=0.95\linewidth]{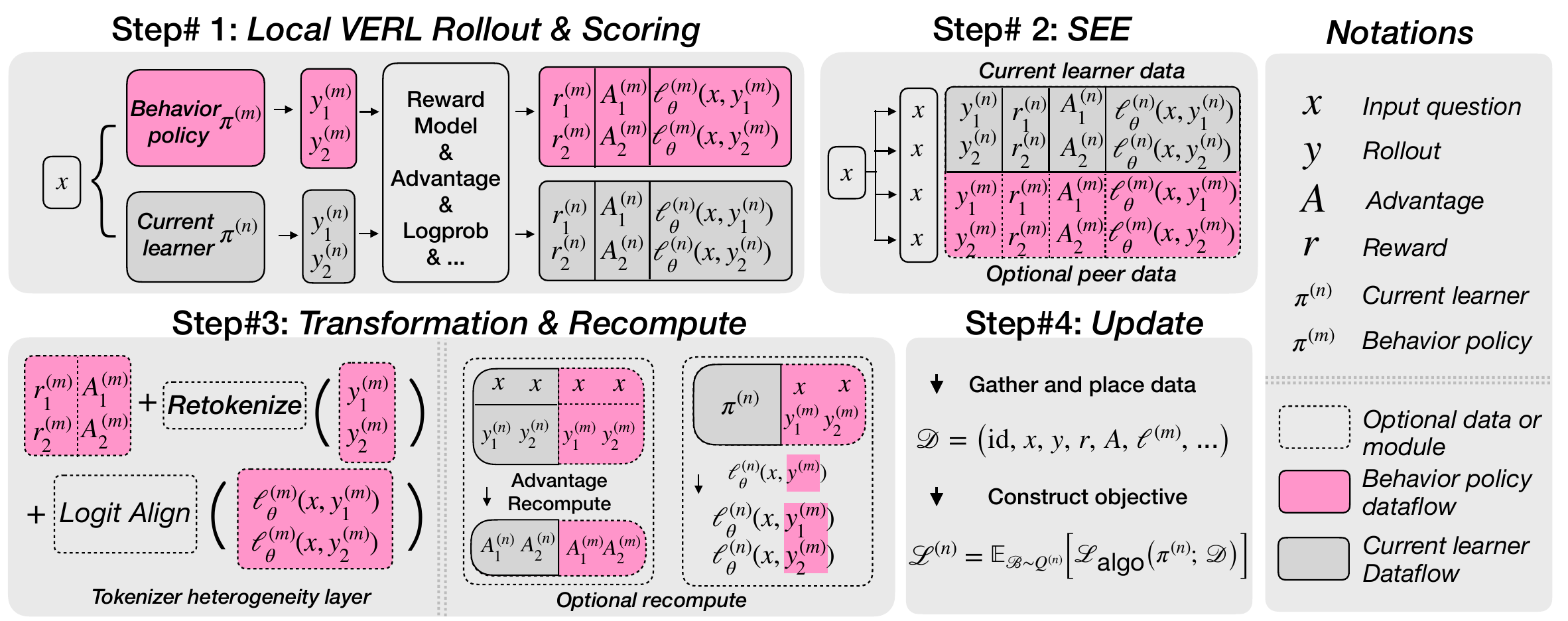}
  \vspace{-2mm} 
  \caption{\textbf{Mutual RL preserves local policy updates while routing typed peer fields through SEE and THL.} Each policy first performs native VERL rollout and scoring, publishes selected fields to the shared exchange, transforms subscribed peer fields into learner-compatible tensors, and then constructs its regime-specific GRPO-based objective independently.}
  \label{fig:overall-dataflow} 
  \vspace{-4mm} 
\end{figure*}

\section{The Mutual RL System Design}
\label{sec:system-design}

Figure~\ref{fig:system-design} summarizes the architecture.
We extend VERL~\citep{sheng2024hybridflow} with a communication layer that preserves native per-policy objectives. Each policy keeps its own rollout, scoring, and learner loop; the shared components only determine which typed fields are published, subscribed to, transformed, and consumed. \SeeFull{} stores the shared pool, \MwraFull{} assigns workers and budgets, and \THL{} converts text and token-level traces across tokenizer boundaries. Fig.~\ref{fig:overall-dataflow} shows the resulting dataflow.

\begin{wrapfigure}{r}{0.45\linewidth}
  \vspace{-\intextsep}
  \centering
  \includegraphics[width=\linewidth]{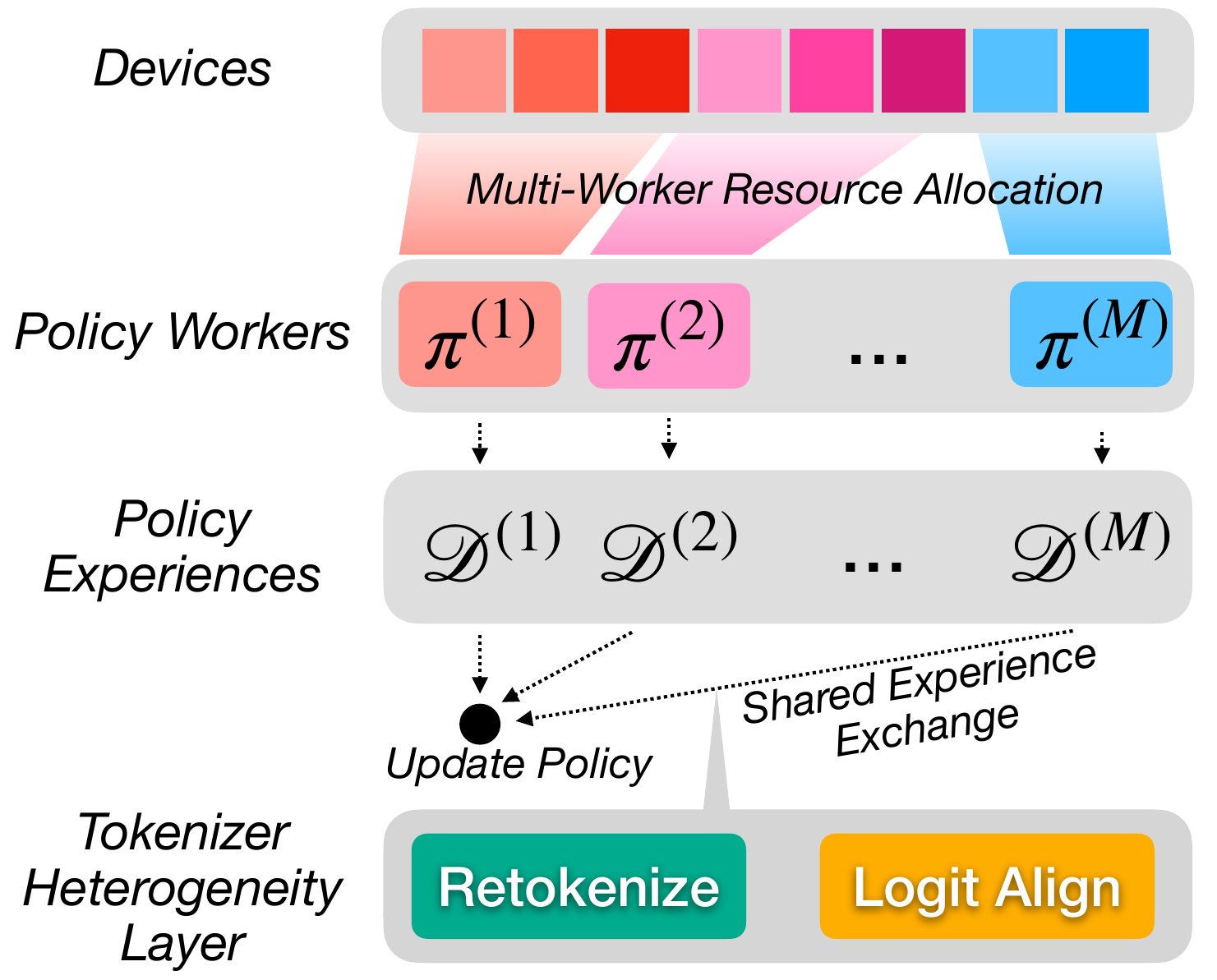}
  \vspace{-1em}
  \caption{\textbf{The system decouples resource placement, experience exchange, and tokenizer alignment.} \MWRA{} assigns devices to policy workers; each policy publishes typed experience $\mathcal{D}^{(m)}$ into \SEE{}; subscribed peer fields are retokenized or trace-aligned by \THL{} before the learner constructs its own update.}
  \label{fig:system-design}
  \vspace{-3em}
\end{wrapfigure}

\subsection{\MwraFull{}}
The \MwraFull{} manages the VERL worker groups responsible for sampling, scoring, and learning. By default, it distributes workers for different policies evenly across devices. Users can override the placement with an explicit device map to pin specific policies to selected devices.

\subsection{Dataflow in One Training Loop}
\paragraph{Standalone computation.} Each policy $m$ runs the standard on-policy generation loop for a batch of prompts $x$. Assigned workers sample candidate trajectories $y^{(m)}_i\sim \pi^{(m)}(\cdot\mid x)$, evaluate rewards $r$, compute local advantages $A$, and cache the per-token log-probability trace $\ell^{(m)}$. This step is policy-local and parallel across policies.
\paragraph{\SeeFull{} (\SEE{}).} The \SEE{} maintains the shared pool of typed experience records $\mathcal{D}=(\mathrm{id}, x_{\text{text}}, y_{\text{text}}, r, A, \ell^{(m)}, \text{meta})$, where $x_{\text{text}}$ and $y_{\text{text}}$ are raw texts, $r$ is the verifier reward, $A$ contains optional advantages, $\ell^{(m)}$ stores behavior-policy traces, and \text{meta} records provenance. The active regime determines which fields are published and which fields peers subscribe to. \MOneShort{} subscribes to rollout text and behavior traces, \MThreeShort{} subscribes to scalar rewards, and \MFourShort{} subscribes to verified successful responses on gated prompts.
\paragraph{Peer data transformation and integration.} Retrieved peer fields are transformed only when the learner needs them. Text fields are retokenized into the learner vocabulary; token-level traces are aligned through \THL{} (Section~\ref{sec:thlmethod}); scalar rewards require no tokenizer processing. The transformed fields are routed according to the active regime: peer trajectories enter the local candidate pool for \MOneShort{}, rewards enter pooled normalization for \MThreeShort{}, and verified peer successes enter the gated auxiliary loss for \MFourShort{}.

\paragraph{Loss construction and update.} After transformation, each policy constructs its configured loss locally and updates its own parameters in parallel. Standard guards such as importance-ratio clipping and group normalization are applied where the regime requires them~\citep{shao2024deepseekmath,Schulman2017PPO}. No additional synchronization is required during the optimization step.

\subsection{\ThlFull{} (\THL{})}
\label{sec:thlmethod}

Heterogeneous policies cannot exchange token-level fields by assuming a shared vocabulary. Token IDs, special-token conventions, and vocabulary versions differ across model families, so each learner must ingest peer data through a tokenizer-aware interface. \THL{} provides that interface through two operations. \emph{(1) Text adaptation (retokenization):} \THL{} decodes the peer response into text and re-encodes it with the learner tokenizer, producing character-token offset mappings for the generated response; because the prompt $x$ is shared across policies, \THL{} retokenizes only the generated response $y$ and appends it to the learner's cached prompt representation. \emph{(2) Policy-distribution alignment (logits projection):} objectives that use peer behavior traces, such as PRP sequence ratios, require a denominator on the learner token grid, and \THL{} constructs a word-level trace alignment in which whitespace-delimited words are used as anchors, source log-probabilities are summed within each word, and the word total is redistributed over the target tokens for that word. This is a trace-alignment primitive, not an exact probability-correction oracle: under tokenization mismatch, the residual enters the importance ratio multiplicatively; App.~\ref{app:theory} (\Cref{cor:thl-ratio}) gives the envelope and App.~\ref{subsec:thl-denominator} reports the empirical ratio statistics.

Mechanistically, we first aggregate the source log-probabilities within each textual word to obtain the total word-level mass \(Z_w = \sum_{s:\, w_{\mathrm{src}}(s)=w} m^{\mathrm{src}}_s\,\ell^{\mathrm{src}}_s\). We then distribute this mass across the corresponding target tokens. Let \(C_w\) be the count of target tokens spanning word \(w\); we assign the aligned log-prob as \(\tilde\ell^{\mathrm{tgt}}_t = m^{\mathrm{tgt}}_t\, Z_{\,w_{\mathrm{tgt}}(t)} / \max(1, C_{\,w_{\mathrm{tgt}}(t)})\) (visualized in Fig.~\ref{fig:wordalign-demo}).

\begin{wrapfigure}{r}{0.5\linewidth}
  \vspace{-\intextsep}
  \centering
  \includegraphics[width=\linewidth]{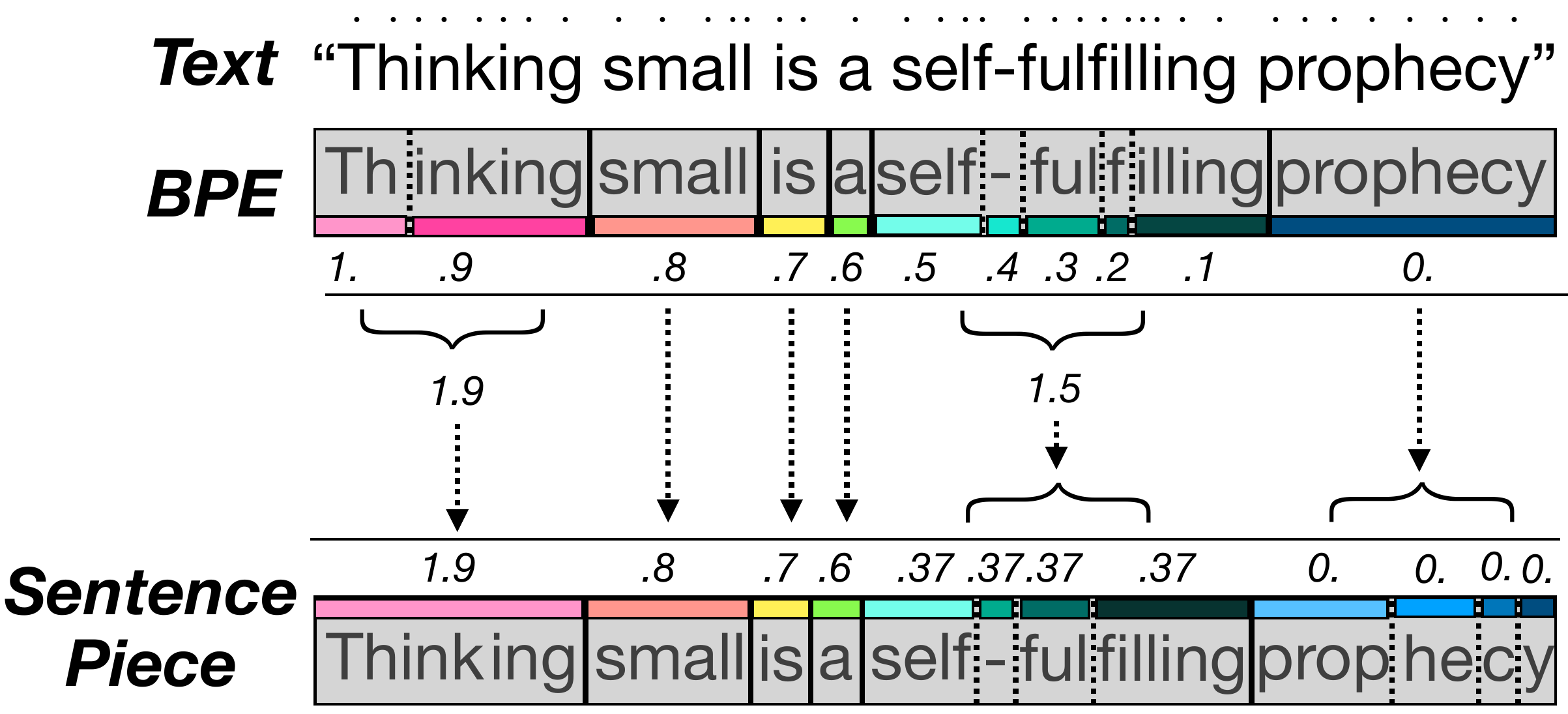}
  \vspace{-1em}
  \caption{\textbf{\THL{} preserves word-level trace mass under tokenizer mismatch.} Source per-token log-probs for \textit{Thinking small is a self-fulfilling prophecy} are summed per textual word ($Z_w$) and evenly assigned to the target tokens of that word ($\tilde\ell_t = Z_w/C_w$).}
  \label{fig:wordalign-demo}
  \vspace{-3em}
\end{wrapfigure}

This rule preserves per-word log-probability mass when source and target segmentations refine the same word partition; mismatched or uncovered spans contribute the residual bounded in \Cref{thm:thl-residual}. The result is a target-aligned probability trace $\tilde{\ell}^{(m\to n)}$ for estimating sequence-level importance ratios $\rho^{(n\mid m)}$ on peer-generated text, with the bias-variance consequences for direct rollout reuse characterized in \Cref{thm:prp-bv}.

\section{Three Controlled Probes of Experience Sharing}
\label{sec:methods}

With the framework in place, we instantiate three controlled probes on top of the same base GRPO update~\citep{shao2024deepseekmath}, occupying the data, value, and outcome levels of experience sharing for matched-compute comparison and structural analysis (App.~\ref{app:theory}). Each is detailed below.

\subsection{Regime 1: Data-level experience sharing via \MOneShort{}}
\label{subsec:prp}

As a controlled probe of data-level experience sharing, the learner pools peer rollouts directly into its own batch and processes them with its standard policy gradient. Here, a learner policy $n$ augments its own generation batch with trajectories produced by its peers.
For a given prompt $x$, rather than optimizing on just its own group of $K$ outputs, the learner constructs a \textbf{unified candidate pool} $\mathcal{Y}_{\text{pool}}^{(n)} = \{y^{(n)}_{1:K}\} \cup \bigcup_{m \neq n} \{y^{(m)}_{1:K}\}$ by concatenating its own and peer generations. The learner then applies the standard GRPO update over this expanded set. Crucially, group-relative advantages $A_y$ are normalized across the entire pool $\mathcal{Y}_{\text{pool}}^{(n)}$, effectively comparing the learner's performance directly against peer baselines.
The critical design choice is the behavior-policy denominator in the importance weights $w^{(n)}_\theta$. Peer data is generated by a different behavior policy $\mu^{(m)}$ (where $m \neq n$), so PRP separates two denominator choices in the weight calculation $w_t = \frac{\pi_\theta(y_t)}{\pi_{\text{behavior}}(y_t)}$:

\begin{enumerate}[topsep=0pt,itemsep=4pt,parsep=0pt,leftmargin=*]
    \item \textbf{Learner-snapshot denominator.}
    This variant treats peer rollouts \emph{as if} they were generated by the learner's own behavior snapshot, setting $\pi_{\text{behavior}} = \pi^{(n)}_{\theta_{\mathrm{old}}}$.
    $$
    w^{(n)}_\theta(x,y) = \exp\big(\ell^{(n)}_\theta(x,y) - \ell^{(n)}_{\mathrm{old}}(x,y)\big),
    $$
    which ignores the peer's actual generation probabilities and relies on the PPO clip in Eq.~\ref{eq:grpo} to limit contributions from mismatched rollouts.

    \item \textbf{\THL{}-aligned peer denominator.}
    This variant uses the peer's recorded likelihood trace after alignment into the learner's token space through \THL{}, setting $\pi_{\text{behavior}} = \mu^{(m)}$ at the trace level.
    $$
    w^{(n)}_\theta(x,y) = \exp\big(\ell^{(n)}_\theta(x,y) - \ell^{(m)}_{\text{aligned}}(x,y)\big).
    $$
    Compared with the learner-snapshot denominator, this accounts for behavior-policy mismatch; App.~\ref{app:theory} analyzes the associated density-ratio variance term and the multiplicative \THL{} alignment-residual envelope.
\end{enumerate}

\subsection{Regime 2: Value-level experience sharing via \MThreeShort{}}
As a controlled probe of value-level experience sharing, the learner samples its own rollouts but shapes their advantages using reward statistics aggregated across the policy pool. This isolates the role of collaborative credit assignment from any direct off-policy signal, with the actor remaining strictly on-policy.

In this setup, \emph{generation is local, but evaluation is global}. Policy $n$ generates and trains solely on its own rollouts $y^{(n)}$. For each prompt $x$, the learner's local group-relative advantage is combined with a pool-aggregated counterpart computed from the per-prompt reward statistics across the learner and its $M-1$ concurrent peers, $\mathcal{R}_{\text{pool}} = \{r^{(n)}\} \cup \bigcup_{m \neq n}\{r^{(m)}\}$. The resulting effective advantage is regularized to discourage long-response domination and clipped before entering the loss; the precise mixing rule, regularizer strength, and clip width are detailed in App.~\ref{app:impl-details} (and the variance condition under which the pooled rule strictly reduces gradient variance is given in \Cref{thm:xgrpo-variance}, App.~\ref{app:theory}).

This mechanism effectively ``grades'' the learner on a population-level curve: if a peer finds a much better solution, the learner's relative advantage on its own rollouts decreases. Crucially, because only scalar reward statistics are exchanged, this regime is tokenizer-agnostic and requires no off-policy corrections.

\subsection{Regime 3: Outcome-level experience sharing via \MFourShort{}}

As a controlled probe of outcome-level experience sharing, \MFourShort{} treats peer outputs as sparse outcome certificates rather than general off-policy rollouts. The interface transfers a single verified peer success to the learner exactly when the learner's own rollout group contains no correct solution. Denote $\mathcal{S}^{(n)}_x$ as the successful rollout collection of policy $n$ under the verifier. \emph{(1) Gate:} the learner has no correct response among its $K$ rollouts ($\mathcal{S}^{(n)}_x = \varnothing$) and at least one peer has a correct solution ($\mathcal{S}^{(-n)}_x \neq \varnothing$); \emph{(2) Transfer:} a successful peer trajectory $y^* \sim \mathrm{Uniform}(\mathcal{S}^{(-n)}_x)$ is injected as a sparse positive sample for the learner via $\mathcal{L}_{\mathrm{SGT}} = - \log \pi_\theta^{(n)}(y^* \mid x)$.

The total loss for policy $n$ combines the on-policy GRPO update with the gated outcome transfer:
\begin{equation}
\label{eq:regime3_total}
\mathcal{L}^{(n)}(\theta) = \mathcal{L}_{\mathrm{GRPO}}^{(n)}(\theta) \;+\; \lambda \cdot \mathbb{I}_{\text{trigger}} \cdot \mathcal{L}_{\mathrm{SGT}}^{(n)}(\theta),
\end{equation}
where $\mathbb{I}_{\text{trigger}}$ is the binary indicator for the gate above and $\lambda$ is a balancing coefficient. \MFourShort{} differs from rejection-sampling fine-tuning~\citep{dong2023raft,xiong2025minimalist} along three axes: it is peer-conditioned (the positive trajectory is drawn from a heterogeneous peer rather than the model's own rollouts), failure-triggered (the gate fires only on the rescue subset where the learner produced no correct sample), and interleaved with on-policy GRPO rather than applied as a standalone offline objective. The auxiliary NLL supplies a missing positive trajectory when the learner's own rollout group contains none, propagating verified solutions across heterogeneous policies during the active RL loop without replacing the local GRPO objective.

\section{Experiments}
\label{sec:experiments}

The main evaluation compares the three experience-sharing probes under matched compute, hyperparameters, and deterministic validation decoding. Section~\ref{subsec:sgt-q2m-q3b} reports the mathematical-reasoning comparison, and Section~\ref{subsec:diagnostics} analyzes the optimization signatures that distinguish the three regimes. The appendix provides the supporting checks that isolate each component: \emph{(1) Infrastructure fidelity:} Appendix~\ref{subsec:verl-sanity} verifies that the VERL-based stack reproduces a standard GRPO baseline, and Appendix~\ref{app:thl-wikitext} verifies \ThlFull{} word-level mass preservation; \emph{(2) Regime controls:} Appendix~\ref{subsec:prp-ablation} isolates the role of importance correction in data-level sharing, Appendix~\ref{subsec:xgrpo-ablation} separates pooled reward statistics from random baseline noise, and Appendix~\ref{app:naive_distillation} compares online gated transfer with sequential offline SFT; \emph{(3) Outcome-level extension:} Appendix~\ref{subsec:sgt-commonsense} evaluates \MFourShort{} on commonsense and scientific QA benchmarks using the same unified answer-extraction template. See Appendix~\ref{app:impl-details} for full experimental setups.

\begin{figure*}[t]
  \centering
  \includegraphics[width=0.95\linewidth]{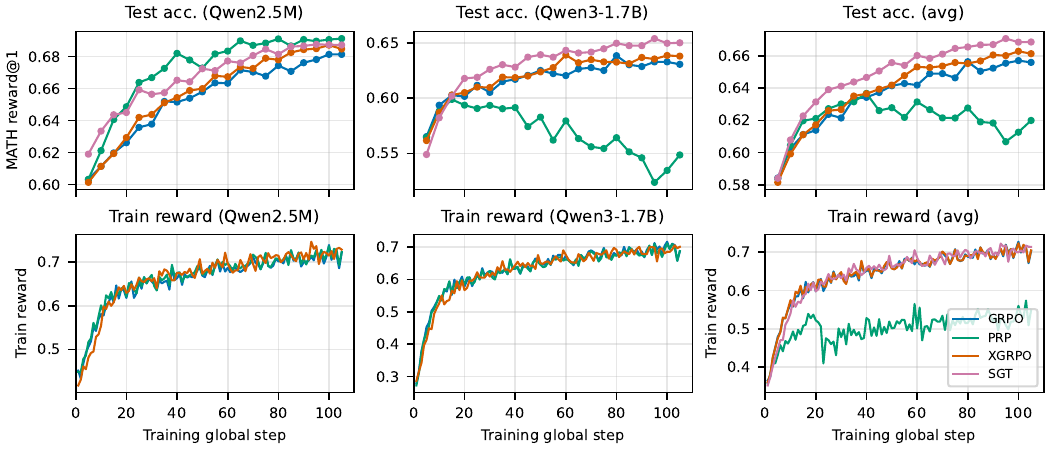}
  \caption{\textbf{Outcome-level sharing gives the best stability-support point in the two-model math comparison.} Validation MATH reward@1 (top) and training reward (bottom) for standalone GRPO, data-level sharing (\MOneShort{}), value-level sharing (\MThreeShort{}), and outcome-level sharing (\MFourShort{}). \MFourShort{} achieves the strongest average performance and faster early-to-mid training progress; \MOneShort{} exposes the sensitivity of direct rollout sharing to policy mismatch; \MThreeShort{} remains a stable low-coupling baseline. For training reward, \MOneShort{} is omitted from per-model panels because its pooling mechanism mixes training rewards for one model; it is shown only in the averaged view.}
  \vspace{-1.5em}
  \label{fig:q2m-q3b-sgt}
\end{figure*}

\subsection{Probe Comparison}
\label{subsec:sgt-q2m-q3b}

We compare the three controlled experience-sharing probes (Sec.~\ref{sec:methods}) against standalone GRPO using a heterogeneous Qwen-family pool that combines a math-specialized Qwen2.5-Math-1.5B with a general-purpose Qwen3-1.7B-Base, spanning two Qwen generations with different pretraining specializations. The Qwen family is the most widely used backbone family for controlled-pool RL post-training comparisons in contemporary verifier-based reasoning work~\citep{yu2025dapo,DBLP:DoesRLIncentivizeReasoning,zheng2025group,1-shot-RL}. Models are trained on MATH and evaluated on the test set. Each curve in Fig.~\ref{fig:q2m-q3b-sgt} is one matched-compute, matched-hyperparameter, deterministic-decoding training trajectory; reported numbers are final-checkpoint values under the same validation protocol. Appendix~\ref{subsec:sgt-commonsense} extends this controlled comparison to four policy pools spanning four open-weights model families (Qwen, Llama, Ministral, Phi-4-mini) on a broader benchmark suite, validating outcome-level transfer across model families and diagnosing the \THL{} alignment path under cross-family tokenizer mismatch (App.~\ref{app:thl-diagnostics}).

\noindent\textbf{Outcome-level sharing (\MFourShort{}) calibrates at a favorable point of the trade-off.}
Among the three probes, outcome-level experience sharing (instantiated by \MFourShort{}) calibrates at the most favorable point of the stability-support trade-off in this setting. On both backbones in the 2Qwen pool, \MFourShort{} stays at or above the standalone GRPO curve at virtually every validation checkpoint of the matched-compute training trajectory (Fig.~\ref{fig:q2m-q3b-sgt}) and reaches the standalone GRPO final-checkpoint accuracy noticeably earlier in training. Per-channel compute overhead is reported in App.~\ref{subsec:sgt-cost} (\MFourShort{} adds $0.0065\times$ one rollout, \MThreeShort{} adds zero, and \MOneShort{} adds $1\times$ rollout). Structurally, \Cref{thm:sgt-rescue-gradient} (App.~\ref{app:theory}) establishes that outcome-level sharing supplies a rescue-set score direction toward verified peer successes that the other two probes do not provide. Because \MFourShort{} only transfers verified peer successes on the rescue subset where the learner produced no correct rollout, the learner's on-policy GRPO update is preserved on every other prompt (the activation profile across pools is reported in App.~\ref{subsec:sgt-where-fires}, and the prompt-specificity of the transferred signal is verified by the matched-vs-mismatched-teacher control of App.~\ref{subsec:sgt-matched-teacher}); the peer signal supplies a missing positive trajectory rather than replacing the local objective. Per-task variation in this calibration tracks the rescue-gate frequency $\mathbb{E}_x[G_n(x)]$ in \Cref{lem:sgt-self-limit} and \Cref{prop:sgt-perturb}: tasks on which the policy pool already attains high success rates produce a smaller rescue subset, and the auxiliary signal correspondingly shrinks.

\noindent\textbf{Value-level sharing (\MThreeShort{}) preserves on-policy sampling and reduces baseline variance.}
\MThreeShort{} stays at or above the standalone GRPO curve through most of training on both backbones. Exchanging only scalar reward statistics keeps the actor strictly on-policy and preserves the sampled support of the base optimization trajectory; structurally, \Cref{thm:xgrpo-variance} gives the condition under which the pooled baseline strictly reduces gradient variance, and \Cref{thm:xgrpo-direct-support} establishes that the value-level probe contributes no direct score-function term for peer-only successful trajectories. Appendix~\ref{subsec:xgrpo-ablation} compares \MThreeShort{} against a baseline that injects random reward perturbations into the per-prompt advantage; the random baseline does not reproduce \MThreeShort{}'s curve, indicating that pooled normalization carries structured cross-policy reward information rather than acting as injected noise.

\noindent\textbf{Data-level sharing (\MOneShort{}) exposes the divergence-sensitivity of direct rollout reuse.}
\MOneShort{} characterizes the high-coupling endpoint of the design space: it transfers complete peer trajectories into the learner's GRPO update. In the two-policy setting, it benefits the stronger policy on parts of the trajectory but imposes a larger optimization burden on the weaker policy, reducing average accuracy. Structurally, \Cref{thm:prp-bv} attributes this behavior to the density-ratio variance term $1+\chi^2(\pi^{(n)}_\theta\,\|\,\mu^{(m)})$ that direct rollout reuse pays under behavior mismatch, with the \THL{}-aligned denominator entering multiplicatively through \Cref{cor:thl-ratio} (the resulting token-level importance-ratio statistics relative to shuffled and broken-alignment controls are reported in App.~\ref{subsec:thl-denominator}). \Cref{prop:prp-anti-align} additionally constructs a baseline-subtracted bandit instance on which naive peer pooling is anti-aligned with the on-policy gradient. The data-level probe therefore motivates the less tightly coupled interfaces in \MThreeShort{} and \MFourShort{}.

\begin{figure*}[t]
  \centering
  \includegraphics[width=0.95\linewidth]{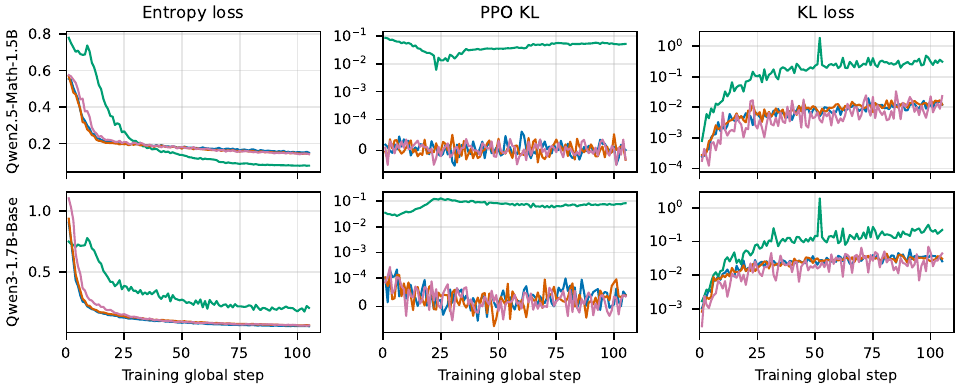}
  \vspace{-4mm}
  \caption{\textbf{The regimes have distinct optimization signatures.} Entropy loss, PPO KL coefficient, and KL loss show that \MFourShort{} preserves entropy comparable to or higher than GRPO while maintaining low KL, \MOneShort{} produces larger KL and lower entropy under direct rollout sharing, and \MThreeShort{} remains close to GRPO because it leaves actor sampling on-policy.}
  \label{fig:q2m-q3b-regularizers}
  \vspace{-4mm}
\end{figure*}

\subsection{Optimization Diagnostics}
\label{subsec:diagnostics}

Fig.~\ref{fig:q2m-q3b-regularizers} shows the optimization signature of each coupling level through entropy and KL traces.

\noindent\textbf{\MFourShort{} preserves entropy under outcome transfer.}
\MFourShort{} maintains entropy comparable to or slightly higher than the GRPO baseline, and the auxiliary outcome-transfer signal keeps KL within the same low range as the baseline. The gate transfers a verified peer success only on a prompt where the learner currently has no correct rollout, so the learner receives outcome-level evidence that complements rather than replaces its own exploration distribution.

\noindent\textbf{\MThreeShort{} preserves the trust region.}
The diagnostic traces for \MThreeShort{} are nearly indistinguishable from the single-policy baseline. By exchanging only scalar rewards and leaving token-level actor gradients on learner samples, \MThreeShort{} reduces variance in the advantage estimate without changing the policy's sampled trajectory support.

\noindent\textbf{\MOneShort{} produces stronger distribution coupling.}
\MOneShort{} drives the stronger policy's entropy below the baseline while the weaker policy's training trace shows larger KL deviations relative to GRPO. Both effects match the data-level channel coupling each learner to a peer-dominated trajectory distribution; the weaker policy in particular receives gradient signal pointing toward trajectories outside its current effective trust region, which is the regime in which direct rollout sharing requires additional distribution control.

\vspace{-2em}
\section{Conclusion}
\vspace{-1em}
\label{sec:conclusion}

Mutual RL turns peer discoveries into typed training signal for heterogeneous LLM policies. (i) The framework substrate (SEE, MWRA, THL) lets policies with different tokenizers exchange rollouts, rewards, advantages, and aligned traces while preserving native per-policy objectives. (ii) Three controlled probes populate the resulting design space on the same GRPO substrate: PRP at the data level, XGRPO at the value level, and SGT at the outcome level. (iii) The contextual-bandit characterization establishes the structural positions on the stability-support trade-off: PRP exposes density-ratio and THL-residual costs, XGRPO preserves on-policy actor support while changing baselines, and SGT supplies a gate-weighted rescue-set score direction. In the evaluated regime, the empirical curves, diagnostics, and theory agree that outcome-level sharing occupies the favorable point of this trade-off.

\bibliographystyle{plainnat}
\bibliography{icml}

\newpage
\appendix

\section*{Ethics Statement}

This work studies reinforcement learning algorithms for large language models in a controlled, multi-policy training setting. We do not collect new data, interact with human subjects, or process personally identifiable or otherwise sensitive information. All benchmarks are public and widely used in prior work. The transfer mechanism operates only through task-level verifier rewards on reasoning benchmarks, and the experiments run in research environments on open-weights base models. Responsible deployment of mutual-RL systems should pair any cross-policy transfer channel with task-appropriate verifiers, safety filters, and the upstream licenses of the models and datasets used.

\section*{LLM Use Statement}
\label{app:llm-use}

We used large language models as general-purpose assistants during this project. Concretely, LLMs were used to help with editing and paraphrasing prose, suggesting alternative phrasings for section titles and abstracts, generating boilerplate code and configuration templates, and checking for obvious inconsistencies in notation and references. All technical content, experimental designs, implementations, and analyses were authored, verified, and run by the authors, and all LLM-generated text and code was manually reviewed and edited before inclusion in the paper. 

\section{Experiment setup and implementation details}
\label{app:impl-details}
\subsection{Experimental Setup}
\label{sec:experimental-setup}

\paragraph{Benchmarks.}
The primary comparison uses \textbf{MATH}~\citep{hendrycks2021math} as the verifiable mathematical-reasoning testbed. The extension study in Appendix~\ref{subsec:sgt-commonsense} adds six public QA benchmarks spanning scientific, physical, and social reasoning: AI2-ARC, OpenBookQA, BoolQ, PIQA, HellaSwag, and Social IQa. All tasks use chain-of-thought prompting with a boxed final answer, giving the verifier a uniform answer-extraction interface across mathematical and multiple-choice settings.

Evaluation is strictly zero-shot: for each prompt we generate a single completion, extract the final answer between the outermost \verb|\boxed{}| tokens, and compare this string exactly against the canonical label. Semantically correct answers with non-matching formatting are counted as incorrect; additional text outside the \verb|\boxed{}| tokens is ignored.

\paragraph{Policy Pools and System Configurations.}
Our pools span 1.5B--4B open-weights backbones from four model families (Qwen, Llama, Ministral, Phi-4-mini), comprising Qwen2.5-Math, Qwen3-Base, Qwen3-4B-Base, Llama-3.2-3B-Instruct, Ministral-3B-Instruct, and Phi-4-mini-Instruct. The five configuration families in Tab.~\ref{tab:main-results} test two axes simultaneously: coordination across multiple policies, and signal exchange across partially or fully mismatched tokenizers from different pretraining families.
Each configuration is trained and evaluated independently.
\begin{itemize}[topsep=0pt,itemsep=0pt,parsep=0pt,partopsep=0pt,leftmargin=*]
    \item \textbf{Homogeneous Pools (2Qwen, 3Qwen, 4Qwen):} We start with a \textbf{2Qwen} pair (Qwen2.5-Math-1.5B~\citep{yang2024qwen2.5math} and Qwen3-1.7B-Base~\citep{yang2025qwen3}) and expand to \textbf{3Qwen} (adding Qwen3-4B-Base~\citep{yang2025qwen3}) and \textbf{4Qwen} (a swarm of four 1.5B/1.7B variants from Qwen2.5/Qwen3~\citep{yang2024qwen2.5math,yang2025qwen3}). These pools test coordination among related models with different capabilities and non-identical tokenizer conventions.
    \item \textbf{Heterogeneous Pools (Llama-Mist, Qwen-Phi4):} The heterogeneous pools stress the \THL{} path directly: \textbf{Llama-Mist} pairs Llama-3.2-3B-Instruct~\citep{dubey2024llama3} with Ministral-3B-Instruct~\citep{mistral2024ministral}, and \textbf{Qwen-Phi4} combines Phi-4-mini-Instruct~\citep{abouelenin2025phi4mini} with Qwen3-4B-Base~\citep{yang2025qwen3}. These configurations require cross-family retokenization and token-level trace alignment.
\end{itemize}

\paragraph{Model-family coverage.}
The five pools collectively span open-weights backbones from four model families (Qwen, Llama, Ministral, Phi-4-mini) with distinct pretraining data, instruction-tuning recipes, and tokenizer vocabularies. The Llama-Mist and Qwen-Phi4 pools each pair two backbones from different families, exercising the \THL{} alignment path on fully disjoint vocabularies.

\begin{table*}[t]
  \centering
  \small
  \caption{\textbf{Policy pools used to test coordination and tokenizer heterogeneity.} The homogeneous Qwen pools vary pool size and capability within a related family, while Llama-Mist and Qwen-Phi4 require cross-family retokenization and trace alignment.}
  \label{tab:main-results}
  \begin{tabular}{lll}
    \toprule
    Config & \# Policies & Backbones \\
    \midrule
    2Qwen & 2 & Qwen2.5-Math-1.5B; Qwen3-1.7B-Base \\
    3Qwen & 3 & Qwen2.5-Math-1.5B; Qwen3-1.7B-Base; Qwen3-4B-Base \\
    4Qwen & 4 & Four Qwen2.5/3 variants (Instruct, Coder, Math; 1.5B/1.7B) \\
    Llama-Mist & 2 & Llama-3.2-3B-Instruct; Ministral-3B-Instruct \\
    Qwen-Phi4 & 2 & Phi-4-mini-Instruct; Qwen3-4B-Base \\
    \bottomrule
  \end{tabular}
\end{table*}

\paragraph{Baselines and Regimes.}
For every backbone, we compare the mutual-RL probes against two single-policy references:
\begin{enumerate}[topsep=0pt,itemsep=0pt,parsep=0pt,partopsep=0pt,leftmargin=*]
    \item[(i)] \textbf{No Train:} Zero-shot accuracy of the pretrained model, corresponding to step 0 in the training curves.
    \item[(ii)] \textbf{Standalone GRPO:} The single-policy baseline fine-tuned in isolation using GRPO~\citep{shao2024deepseekmath}.
    \item[(iii)] \textbf{Mutual RL Regimes:} Data-level sharing via \textbf{\MOneShort{}} (Regime 1), value-level sharing via \textbf{\MThreeShort{}} (Regime 2), and outcome-level sharing via \textbf{\MFourShort{}} (Regime 3).
\end{enumerate}

\subsection{Implementation details}

\paragraph{Framework \& Compute.}
All experiments are implemented in VERL~\citep{sheng2024hybridflow} extended with our \SEE{} and \THL{} modules. Standalone and mutual-RL runs use the same training schedules, rollout counts, prompt and response limits, and validation decoding, so each regime comparison isolates the sharing channel rather than a compute or decoding change. Here ``matched compute'' refers to per-policy training schedule, rollout count, optimizer settings, and validation decoding being identical between Standalone GRPO and the mutual-RL variants; regime-specific auxiliary cost (e.g., \MFourShort{}'s sparse NLL pass and \MOneShort{}'s peer-rollout re-scoring) is reported separately in App.~\ref{subsec:sgt-cost} and is not bundled into the per-policy training budget.
All runs use a single Amazon P4 instance with eight A100 GPUs or a single Amazon G6 instance with eight L40S GPUs, with the same hardware for standalone GRPO and mutual-learning variants. Each policy is trained with FSDP on one of the eight GPUs;
\paragraph{Hyperparameters.}
We follow the matched-compute, deterministic-decoding reporting convention used in contemporary large-scale RL post-training of LLMs and diffusion models~\citep{yu2025dapo,li2025tothink,li2026branchgrpo,ding2026treegrpo}: each configuration is trained once under matched compute and hyperparameters, validation uses temperature 0, and reported numbers are taken from the final checkpoint of each training trajectory.
We train for fifteen epochs, with each GRPO update seeing a global batch of \num{1024} prompts, split into four mini-batches of \num{256} prompts and micro-batches of two prompts per GPU (sixteen prompts per micro-batch across eight GPUs). Optimization uses AdamW with learning rate \(1\times10^{-6}\), weight decay \(0.01\), gradient-norm clipping at~1, and a cosine scheduler without warmup; entropy regularization is disabled (coefficient zero), checkpoints are saved every 20 steps, validation runs every 5 steps, and the trainer is run in deterministic mode.
For the loss we use critic-free GRPO with group-relative advantages: for each prompt we sample a group of 5 candidate responses, compute scalar rewards, subtract the group mean to form advantages, and normalize them to zero mean and unit variance across the batch. The discount factor and generalized-advantage parameter are both one, so advantages depend only on per-response returns. KL penalties are applied as a separate loss term rather than folded into rewards; the per-token KL between policy and reference is scaled by a fixed coefficient of \(10^{-3}\). During training we truncate prompts at 1\,024 tokens and responses at 3\,072 tokens, generate up to five responses per prompt with temperature 1 and full nucleus sampling (\(p=1\)), and cap each rollout batch at 8\,192 tokens and 1\,024 sequences; log-prob computations allow sequences up to 16\,384 tokens per GPU, and the rollout engine uses at most 90\% of GPU memory. For validation we decode a single completion per prompt with temperature 0 (no sampling). Data for MATH and the commonsense benchmarks is converted to a unified chain-of-thought multiple-choice format (App.~\ref{app:prompt-templates}); examples whose tokenized prompts exceed 1\,024 tokens are dropped; rewards are computed by exact answer checking for MATH and by matching the correct option for multiple-choice tasks, separately for each policy.

The three sharing mechanisms modify different fields of this GRPO update. Peer Rollout Pooling (\MOneShort{}) aggregates trajectories from multiple policies and feeds each learner a mix of self and peer rollouts, downweighting off-policy contributions with an importance ratio clipped to a band of width 0.2 around one. Cross-Policy GRPO Advantage Sharing (\MThreeShort{}) averages per-prompt advantages across policies, mixes them back into each policy's own advantages with a cross-model factor of 0.2, applies a length correction of strength 0.1, and clips the resulting effective advantages before the loss. Success-Gated Transfer (\MFourShort{}) adds a failure-conditioned outcome-transfer NLL on verified peer successes: responses with rewards above 0.8 are successes, responses below 0.2 are negatives, and the auxiliary loss is scaled by $\lambda_{\mathrm{SGT}}=0.1$ relative to the main GRPO loss; if a batch contains no high-reward examples, the update is pure GRPO. Throughout, we disable the neural reward-model subsystem and rely only on verifiable benchmark-based rewards, so observed effects come from GRPO and the mutual-learning mechanisms rather than from a learned reward model.

\section{Implementation details for Success-Gated Transfer}
\label{app:sgt-details}

\MFourShort{} is implemented as a shared post-processing module on top of the VERL dataflow. The module runs once per training step after rewards (and, optionally, advantages) have been computed for each policy, and it attaches a peer success only for learner-prompt pairs on the rescue subset: the learner has no correct rollout and at least one peer has a verified success. This makes the auxiliary signal peer-conditioned, failure-triggered, and interleaved with the learner's GRPO update rather than an offline reward-ranked fine-tuning pass~\citep{dong2023raft}. Fig.~\ref{fig:q2m-q3b-sgt-entropy} shows that, on the 2Qwen MATH setup, \MFourShort{} preserves entropy relative to the standalone GRPO baseline. The module input is a list of per-policy batches, each containing tokenized prompts and responses together with (i) a binary reward $\tilde r\in\{0,1\}$ for every response and (ii) a prompt identifier shared across models for the same input.

\paragraph{Gating on failure and peer success.} The post-processor groups examples across policies by the shared prompt identifier, recording which models produced successful ($\tilde r=1$) and unsuccessful ($\tilde r=0$) responses. The experiments use binary verifier rewards; scalar verifier scores are thresholded before entering this module. For policy $n$, a prompt is eligible only when the learner's response is negative, $\tilde r^{(n)}_{x,i}=0$, and the stricter configuration used in our experiments requires all sampled responses from that policy on the prompt to be negative. The code then collects \emph{peer positives} (responses with $\tilde r^{(m)}_{x,j}=1$ from models $m\neq n$) and selects a single peer success $y_x^{\star}$ by random choice or by the configured shorter-response heuristic. A per-prompt counter caps how many SGT pairs can be formed for the same prompt, preventing a small set of hard prompts from dominating the auxiliary loss.

\paragraph{Packaging positive examples.} For every learner example that passes the gate, the implementation marks that example as using \MFourShort{} in a boolean mask. The selected peer success is decoded with the peer tokenizer and re-encoded with the learner tokenizer; the learner's original prompt tokens are reused, and the re-encoded positive response is appended to form an auxiliary sequence that the learner can score. These auxiliary sequences and their attention masks are attached to the learner's batch alongside the existing on-policy rollouts, and per-model hyperparameters such as the \MFourShort{} weighting coefficient are recorded in the batch metadata. Examples outside the rescue subset carry inactive masks and dummy placeholders so that batch shapes remain consistent.

\paragraph{Loss computation and combination with GRPO.} During the learner update, \MFourShort{} is evaluated only on examples whose SGT mask is active. For each such example, the learner re-computes log-probabilities on its own unsuccessful response $y^{-}$ and on the selected peer success $y^{\star}$. In the mode used in our experiments, the auxiliary loss is a sequence-level negative log-likelihood on the peer response, averaged per token and aggregated over all gated examples:
\[
\mathcal{L}_{\text{SGT}}^{(n)}
  = -\,\mathbb{E}_{x,i:\,\text{SGT mask}=1}
      \ell^{(n)}_{\theta}\!\big(x,y^{\star}_{x,i}\big),
\]
with standard masking and numerical clipping applied inside the implementation. The learner's total objective for policy $n$ is the sum of the base GRPO loss on on-policy rollouts and a weighted SGT term,
\[
\mathcal{L}_{\text{total}}^{(n)}
  = \mathcal{L}_{\text{GRPO}}^{(n)}
    + \lambda_{\text{SGT}}\,
      \mathcal{L}_{\text{SGT}}^{(n)},
\]
where $\lambda_{\text{SGT}}$ controls the strength of the rescue-set outcome transfer. The gate activates only when a learner fails on a prompt that some peer has solved, and at most one peer success is consumed per prompt and learner. The auxiliary loss therefore supplies a sparse positive trajectory missing from the learner's rollout group while the base GRPO update remains on-policy.

\begin{figure}[t]
  \centering
  \includegraphics[width=0.4\linewidth]{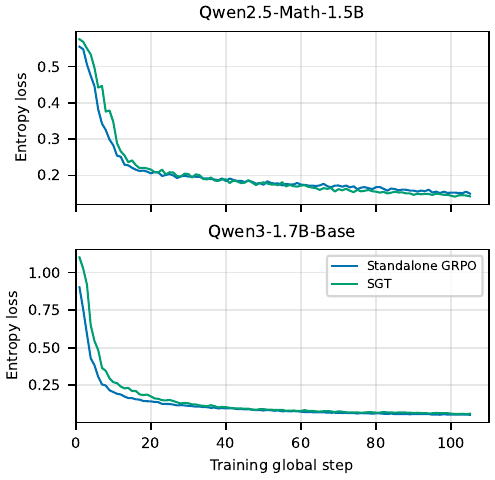}
  \caption{\textbf{SGT preserves entropy while adding rescue-set outcome transfer.} Entropy loss for the standalone GRPO baseline and SGT on Qwen2.5-Math-1.5B (top) and Qwen3-1.7B-Base (bottom). SGT keeps entropy at least as high as the GRPO baseline throughout training, consistent with a peer-only auxiliary loss that fires only on learner-failure prompts rather than replacing the on-policy GRPO update.}
  \label{fig:q2m-q3b-sgt-entropy}
\end{figure}

\section{Additional experimental results}
\label{app:extra-experiments}

\subsection{Infrastructure fidelity: reproducing VERL GRPO}
\label{subsec:verl-sanity}

Before comparing mutual-learning mechanisms, we verify that the VERL-based stack preserves standard single-policy GRPO dynamics. We run two matched Qwen2.5-Math-1.5B experiments on MATH: one launched through the original VERL codebase and one launched through our mutual-RL system built on top of VERL. Both runs use the same GRPO hyperparameters and data pipeline.

The fidelity check runs on a single Amazon G6e instance with eight NVIDIA L40S GPUs. To match the hardware shape used in the main multi-model experiments, we restrict the experiment to effectively use four GPUs, corresponding to two policies sharing the same budget. The comparison therefore isolates implementation fidelity rather than extra compute.

Fig.~\ref{fig:verl-verlhybrid-math} summarizes the comparison. The top-left panel plots test accuracy (MATH reward@1) over training steps, and the top-right panel shows the corresponding training reward. The bottom-left panel reports training throughput in tokens per second, measured using VERL's built-in logging utilities and computed only on non-evaluation steps, while the bottom-right panel tracks the policy-gradient loss. Across all four views, the two runs are closely aligned: test and train accuracy curves almost overlap, throughput is the same order of magnitude, and the PG-loss trajectories exhibit similar trends.

This experiment verifies that the mutual-RL stack reproduces the reference single-policy GRPO pipeline without introducing instability or large slowdowns. All subsequent baselines and mutual-learning variants are run within this unified stack, so regime comparisons share identical logging, data, and hardware conditions.

\begin{figure}[t]
  \centering
  \includegraphics[width=.7\linewidth]{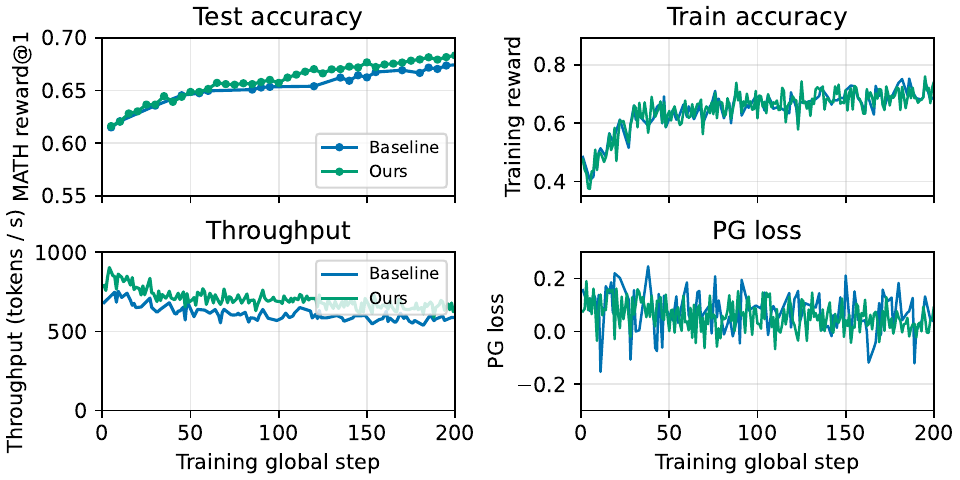}
  \caption{\textbf{System sanity check on single‑policy GRPO.} Comparison between a reference VERL implementation and our VERL‑based mutual‑RL stack on Qwen2.5‑Math‑1.5B trained on MATH. Panels show (top‑left) test accuracy (MATH reward@1), (top‑right) training reward, (bottom‑left) training throughput on non‑evaluation steps, and (bottom‑right) policy‑gradient loss as a function of training step. The curves remain well aligned across all metrics, indicating that our system faithfully reproduces the baseline GRPO dynamics while introducing only minor variations attributable to stochasticity.}
  \label{fig:verl-verlhybrid-math}
\end{figure}
 
\subsection{Sanity checks on THL word alignment}
\label{app:thl-wikitext}

\begin{table}[t]
  \centering
  \small
  \caption{\textbf{\THL{} preserves source word-level mass on WikiText-103.} For each source$\rightarrow$target model pair we report per-word mass errors between source and THL-aligned sequences (Src$\to$Aligned MAE / Max), between THL-aligned and target sequences (Aligned$\to$Target MAE / Max), and the maximum absolute prefix deviation (Prefix Leak Max), aggregated over \num{1024} spans.}
  \label{tab:thl-wikitext}
  \begin{tabular}{lrrrrrr}
    \toprule
    Pair & \multicolumn{2}{c}{Src$\to$Aligned} & \multicolumn{2}{c}{Aligned$\to$Target} & Prefix Leak \\
    \cmidrule(lr){2-3} \cmidrule(lr){4-5}
         & MAE & Max & MAE & Max & Max \\
    \midrule
    Qwen3-1.7B $\rightarrow$ Qwen3-1.7B & 0. & 0. & 0. & 0. & 0. \\
    Qwen2-1.5B $\rightarrow$ Qwen2-1.5B & 0. & 0. & 0. & 0. & 0. \\
    Llama3.2-1B $\rightarrow$ Llama3.2-1B & 0. & 0. & 0. & 0. & 0. \\
    Qwen3-1.7B $\rightarrow$ Llama3.2-1B & 0. & 0. & 12.7053 & 122.7694 & 0. \\
    \bottomrule
  \end{tabular}
\end{table}

To validate that \THL{} preserves word-level probability mass when retokenizing between models, we run a diagnostic on WikiText-103~\cite{merity2016wikitext}. We sample \num{1024} contiguous text spans from the training split of the \texttt{wikitext-103-raw-v1} dataset. Each span contains between 5 and 80 tokens under the Qwen3 tokenizer so that all models see comparable prompts. For each span we compute per-token log-probabilities under three causal language models (Qwen3-1.7B, Qwen2-1.5B, and Llama-3.2-1B) using right-padded batches and a single forward pass per model.

Given a source-target pair, we feed the source log-probabilities and masks through \THL{} to obtain a target-tokenized aligned trace. For each text and each whitespace-delimited word we sum the token log-probabilities belonging to that word, obtaining three per-word totals: source, THL-aligned, and target. The \emph{Src$\to$Aligned} metrics test conservation of the source model's word-level mass; the \emph{Aligned$\to$Target} metrics measure the expected difference between source-derived scores and the target model's native logits. Prefix Leak Max is the maximum absolute prefix sum of source-aligned word differences and detects drift across word boundaries.

Tab.~\ref{tab:thl-wikitext} summarizes the results. In the self-alignment rows (Qwen3$\rightarrow$Qwen3, Qwen2$\rightarrow$Qwen2, Llama3.2$\rightarrow$Llama3.2), both Src$\to$Aligned and Aligned$\to$Target MAE/Max are numerically zero. For the cross-model case (Qwen3-1.7B$\rightarrow$Llama-3.2-1B), Src$\to$Aligned remains zero: \THL{} preserves the Qwen3 word totals. Aligned$\to$Target MAE/Max is large because Qwen3-derived scores need not equal Llama-3.2 native logits. Prefix Leak Max is zero in all rows, showing no cumulative drift across word boundaries.

\begin{algorithm}[t]
\small
\caption{THL: PerTokenWordMapBySegments}
\label{alg:thl-word-map}
\begin{algorithmic}[1]
\REQUIRE $text$; tokenizer (no specials)
\ENSURE $(full\_ids,\ word\_map)$ where $full\_ids$ is tokenization of $text$ (no specials)
\ENSURE $word\_map$ matches $full\_ids$ in length with word index per token or \texttt{None} for leading delimiters
\STATE $full\_ids \leftarrow$ Tokenize($text$, add\_special\_tokens=false)
\STATE spans $\leftarrow$ WordSpans($text$)
\STATE $word\_map \leftarrow [\;]$; $built\_ids \leftarrow [\;]$; $prev\_end \leftarrow 0$
\FOR{$w\_idx, (start, end)$ in enumerate(spans)}
  \STATE $seg \leftarrow text[prev\_end{:}end]$
  \STATE $seg\_ids \leftarrow$ Tokenize($seg$, add\_special\_tokens=false)
  \STATE Append $seg\_ids$ to $built\_ids$
  \IF{$prev\_end = 0$ and $start > 0$}
    \STATE $lead \leftarrow text[:start]$
    \STATE $n\_lead \leftarrow$ Len(Tokenize($lead$, add\_special\_tokens=false))
    \STATE Append $n\_lead$ entries of \texttt{None} to $word\_map$
    \STATE Append $(\text{Len}(seg\_ids) - n\_lead)$ entries of $w\_idx$ to $word\_map$
  \ELSE
    \STATE Append $\text{Len}(seg\_ids)$ entries of $w\_idx$ to $word\_map$
  \ENDIF
  \STATE $prev\_end \leftarrow end$
\ENDFOR
\IF[best-effort alignment]{$\text{Len}(built\_ids) \ne \text{Len}(full\_ids)$}
  \STATE $L \leftarrow \min(\text{Len}(built\_ids),\ \text{Len}(full\_ids))$
  \STATE Truncate $built\_ids \leftarrow built\_ids[:L]$; $word\_map \leftarrow word\_map[:L]$; $full\_ids \leftarrow full\_ids[:L]$
\ENDIF
\STATE \textbf{return} $(full\_ids,\ word\_map)$
\end{algorithmic}
\end{algorithm}

\subsection{Role of Importance Correction in Data-Level Sharing}
\label{subsec:prp-ablation}

We ablate the behavior-policy denominator inside \MOneShort{} on the 2Qwen setup. We compare two instantiations: (1) \textbf{Naive Pooling}, which treats peer rollouts as if they were generated by the learner's own behavior snapshot, and (2) \textbf{Importance-Corrected Pooling}, which uses the peer likelihood trace aligned into the learner token space by \THL{} and applies PPO clipping. Fig.~\ref{fig:prp-ablation} reports validation accuracy.

\paragraph{Effect of importance correction.}
On Qwen2.5-Math-1.5B, removing importance correction reduces final accuracy from 69\% to 58\%. On Qwen3-1.7B-Base, the uncorrected variant tracks below the corrected curve and drifts from a peak near 54\% down to 46\% as training progresses. Treating off-policy peer trajectories without distribution control produces gradients that are poorly matched to the learner's effective trust region, making importance correction the baseline requirement for data-level sharing.

\paragraph{System validation of \THL{}.}
The Qwen2.5 and Qwen3 families use different tokenizers, so the corrected denominator is computed through \THL{}-aligned peer traces. The corrected curve closes the gap on Qwen2.5 and stabilizes Qwen3, showing that the \THL{}-aligned denominator is usable across tokenizer families.

\paragraph{Reading the result.}
Importance correction is the floor for direct rollout sharing, not the ceiling: it prevents divergence but does not by itself reach the accuracy of outcome-level sharing (Fig.~\ref{fig:q2m-q3b-sgt}). The data-level channel therefore serves as a stress test for the design space and motivates the less tightly coupled interfaces in \MThreeShort{} and \MFourShort{}.

\begin{figure}[t]
  \centering
  \includegraphics[width=.6\linewidth]{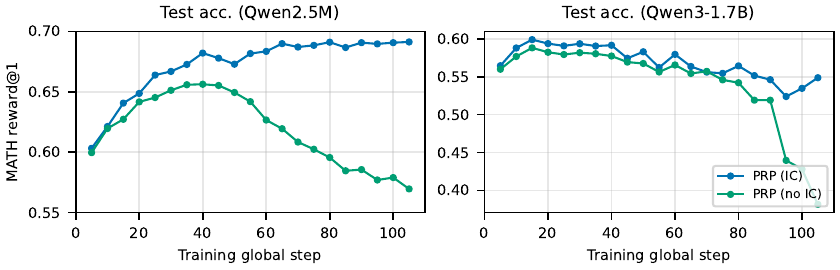}
  \caption{\textbf{\THL{}-aligned importance correction is required for tractable data-level sharing.} Validation MATH reward@1 for \MOneShort{} with and without the peer-denominator correction. The naive-pooling curves (light green) lie below the importance-corrected curves (dark green) on both policies, with a larger gap on the weaker Qwen3-1.7B-Base policy.}
  \label{fig:prp-ablation}
\end{figure}

\subsection{Ablation: shared normalization for XGRPO}
\label{subsec:xgrpo-ablation}

Cross-Policy GRPO Advantage Sharing (\MThreeShort{}) normalizes critic-free advantages using reward statistics pooled across policies, whereas the standalone GRPO baseline computes group-normalized advantages independently for each model. To disentangle the effect of cross-policy normalization from generic baseline noise, we compare three regimes on the same 2Qwen MATH configuration used in Sec.~\ref{subsec:sgt-q2m-q3b}: standalone GRPO, XGRPO, and a GRPO variant where each model still uses its own rollouts but we inject random perturbations into the per-prompt baseline used for advantage computation (the \texttt{id2mean} values) before normalization. Fig.~\ref{fig:xgrpo-ablation} reports validation MATH reward@1 for both math backbones.

\begin{figure}[t]
  \centering
  \includegraphics[width=.6\linewidth]{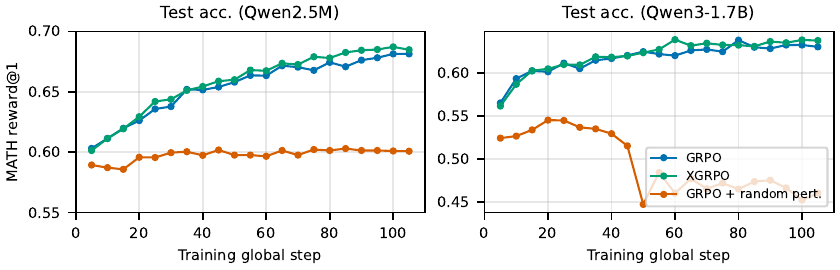}
  \caption{\textbf{XGRPO gains come from structured cross-policy normalization.} Validation MATH reward@1 for standalone GRPO, XGRPO, and a GRPO + random-perturbation baseline (GRPO + random pert.), which injects noise into the per-prompt baseline before computing advantages. Panels show the Qwen2.5-Math-1.5B policy (left) and the Qwen3-1.7B-Base policy (right). XGRPO matches or exceeds GRPO on both models, while GRPO + random pert. lags behind.}
  \label{fig:xgrpo-ablation}
\end{figure}

Across both backbones, XGRPO tracks the GRPO baseline closely in the early stages of training and attains higher final MATH accuracy, matching the value-level pattern in Fig.~\ref{fig:q2m-q3b-sgt}. In contrast, GRPO + random pert.\ is worse than both GRPO and XGRPO, especially on Qwen3-1.7B-Base, where noisy baselines slow convergence and degrade the final score. The control shows that XGRPO leverages informative cross-policy reward structure rather than acting as an implicit regularizer reproduced by random baseline noise.

\subsection{Value-Level Sharing Across Model Pools}
\label{subsec:xgrpo-extended}

We extend the evaluation of \MThreeShort{} (Regime 2) from the 2Qwen setup to three additional configurations from Tab.~\ref{tab:main-results}: \textbf{4Qwen} (a homogeneous swarm of four 1B-level variants), \textbf{3Qwen} (mixing 1B- and 4B-level backbones), and \textbf{Qwen-Phi4} (a heterogeneous Phi-4 / Qwen3 pairing). For each pool we track the average validation MATH reward@1, with training data, hyperparameters, and compute budgets strictly matched between Standalone GRPO and \MThreeShort{}.

Fig.~\ref{fig:xgrpo-ext-avg} summarizes the learning dynamics. Across all three configurations, \MThreeShort{} tracks Standalone GRPO closely and produces final-accuracy improvements in the pools where pooled scalar rewards add useful baseline information.

\paragraph{Interpreting stable value-level sharing.}
\MThreeShort{} remains close to standalone GRPO across larger and more heterogeneous pools because it exchanges only scalar reward statistics and preserves the base actor sampling distribution. This makes value-level sharing the low-coupling point of the design space: it can improve baseline variance without introducing peer trajectories or verified successes. Outcome-level transfer (\MFourShort{}) supplies the direct support-expanding signal; \MThreeShort{} complements that channel by stabilizing value estimates.

\begin{figure}[t]
  \centering
  \includegraphics[width=.75\linewidth]{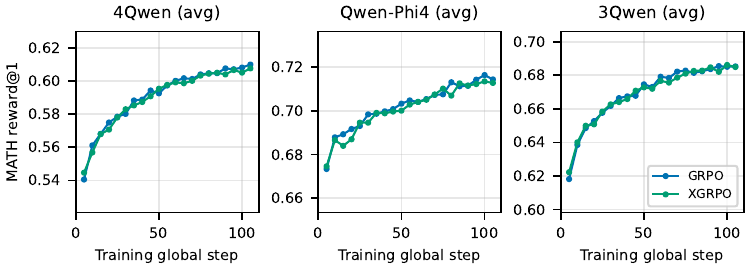}
  \caption{\textbf{Value-level sharing preserves the GRPO stability profile across model pools.} Averaged validation MATH reward@1 for Standalone GRPO vs.\ \MThreeShort{} across three pools. The overlapping curves show that pooled scalar reward sharing remains a low-coupling intervention as pool size and heterogeneity increase.}
  \label{fig:xgrpo-ext-avg}
\end{figure}

\subsection{Online Gating versus Sequential Peer Distillation}
\label{app:naive_distillation}

We compare \MFourShort{} against \textbf{Sequential Peer Distillation}, a control that removes the two structural ingredients of outcome-level transfer: online failure gating and interleaving with the learner's GRPO update.

\paragraph{Setup.}
The offline control is a ``Plain RL $\to$ SFT'' pipeline:
\begin{enumerate}
    \item \textbf{Phase 1 (Data Collection):} Two models (Qwen2.5-Math and Qwen-3-1.7B) are trained independently via standard single-policy GRPO to convergence, matching the Standalone GRPO setup. All distinct successful trajectories generated during their MATH-train runs are harvested.
    \item \textbf{Phase 2 (Peer SFT):} The successful trajectories from the \emph{peer} model are used to fine-tune the \emph{learner} via supervised fine-tuning.
\end{enumerate}

\textbf{SFT hyperparameters.} Learning rate 2e-5 with cosine schedule and 0.03 warmup; LoRA (rank=16, alpha=32, dropout=0.05); 1 epoch. The SFT loss matches the \MFourShort{} auxiliary NLL but is applied offline, without the rescue gate or the concurrent on-policy GRPO update.

\begin{figure}[h]
\centering
\includegraphics[width=0.5\linewidth]{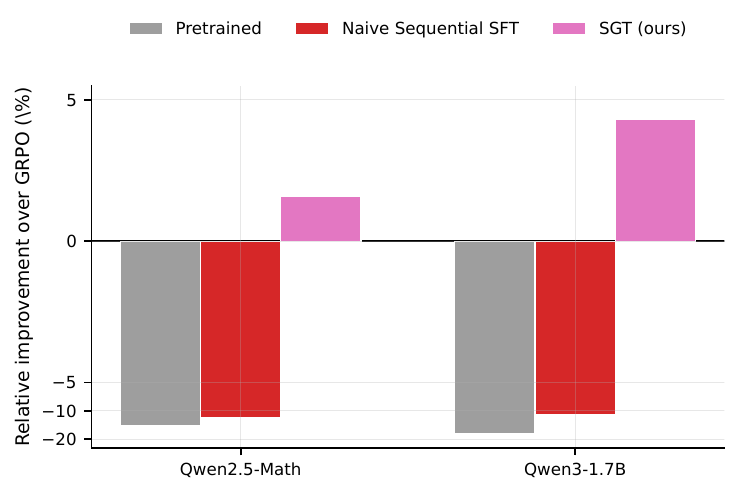}
\caption{\textbf{Online gating versus sequential offline distillation on MATH.} Each bar shows the relative MATH-accuracy improvement over the Standalone GRPO baseline (0 line) on Qwen2.5-Math-1.5B and Qwen3-1.7B-Base. Naive Sequential SFT fine-tunes on peer successes after standalone RL completes and recovers only part of the deficit relative to the pretrained backbone, leaving it well below the RL baseline; \MFourShort{} (online gating, concurrent with GRPO) is the only treatment that exceeds Standalone GRPO on both backbones, with the larger margin on Qwen3-1.7B-Base. Relative to the pretrained backbones, Standalone GRPO contributes $+16.9\%$ and $+21.5\%$ MATH improvement, and \MFourShort{} contributes $+17.9\%$ and $+25.3\%$ on Qwen2.5-Math and Qwen3-1.7B-Base respectively.}
\label{fig:naive_sft}
\end{figure}

\paragraph{Why online gating matters.}
Sequential SFT trails Standalone GRPO by 8.1 points on Qwen2.5-Math and 6.9 points on Qwen-3-1.7B; \MFourShort{} matches or improves on Standalone GRPO under the same budget. Two design factors drive this difference:

\begin{itemize}
    \item \textbf{Concurrent on-policy update.} \MFourShort{} interleaves the auxiliary peer loss with the learner's GRPO step, so the KL constraint of the on-policy update keeps the learner near its own reasoning distribution while the peer trajectory supplies a sparse positive sample. Sequential SFT removes this anchor and pushes the learner toward the peer's full output distribution.
    \item \textbf{Failure-conditioned gating.} \MFourShort{} only activates when the learner's own rollout group contains no correct sample on the prompt. Sequential SFT applies the peer trace on every prompt, including ones where the learner already has a correct trajectory, replacing the learner's own positive signal with the peer's.
\end{itemize}

The interface choice (peer-conditioned, failure-triggered, and interleaved with on-policy GRPO) separates outcome-level transfer from offline distillation.
\subsection{MATH and Commonsense Reasoning with Regime 3}
\label{subsec:sgt-commonsense}

We extend the evaluation of \MFourShort{} from the 2Qwen MATH configuration to a broader benchmark suite. We combine \textbf{MATH}~\citep{hendrycks2021math} with six datasets spanning scientific, physical, and social reasoning: \textbf{AI2-ARC}~\citep{clark2018arc}, \textbf{OpenBookQA}~\citep{mihaylov2018openbookqa}, \textbf{BoolQ}~\citep{clark2019boolq}, \textbf{PIQA}~\citep{bisk2020piqa}, \textbf{HellaSwag}~\citep{zellers2019hellaswag}, and \textbf{Social IQa}~\citep{sap2019socialiqa}. We aggregate these training sets into a unified mixture and evaluate on their respective test sets. We run a focused 50-step training with a global batch size of 1,024 inputs per step.

To support automated reward verification, all inputs use a unified prompt template (Appendix~\ref{app:prompt-templates}). The same template is applied to Standalone GRPO and \MFourShort{}, so the comparison measures the gain from outcome-level transfer under identical decoding and answer-extraction conditions; pretrained accuracy serves as a qualitative reference under the same template.

Fig.~\ref{fig:sgt-step50-commonsense} reports accuracy for No-Train, Standalone GRPO, and \MFourShort{} across four policy pools: \textbf{3Qwen}, \textbf{4Qwen}, \textbf{Llama-Mist}, and \textbf{Qwen-Phi4}.

\paragraph{Robust across diverse reasoning tasks.}
\MFourShort{} matches or improves on Standalone GRPO on most tasks in the homogeneous 3Qwen and 4Qwen pools. In the heterogeneous Qwen-Phi4 configuration, \MFourShort{} attains the highest score on most tasks, with gains of $+2$ to $+3$ percentage points on HellaSwag and Social IQa. Outcome-level transfer therefore extends beyond MATH to broader commonsense domains under the same verifier-driven interface.

\paragraph{Larger gains under heterogeneous pools.}
The largest gains appear on the \textbf{Llama-Mist} configuration, where \MFourShort{} adds $+10$ to $+30$ points over Standalone GRPO on several datasets. Llama-3.2 and Ministral-3B traverse distinct regions of the solution space, and \MFourShort{} routes verified successes between them on prompts where one model fails. This setup also exercises the \THL{} pipeline across fully disjoint vocabularies, since the two backbones share no tokenizer.

\paragraph{Performance under unified-template evaluation.}
Several base models start from low No-Train scores under the unified template (Appendix~\ref{app:prompt-templates}). The gain from \MFourShort{} over Standalone GRPO is computed under matched decoding and answer extraction, isolating the channel comparison from template effects.

\begin{figure*}[t]
  \centering
  \begin{minipage}{0.95\linewidth}
    \centering
    \includegraphics[width=\linewidth]{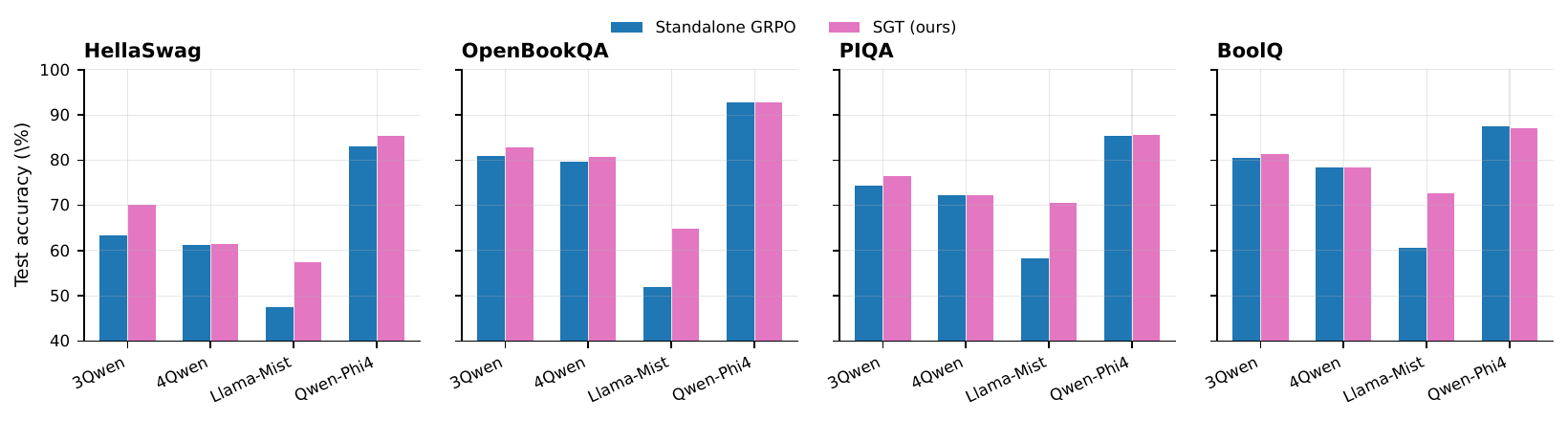}
  \end{minipage}\\[0.4em]
  \begin{minipage}{0.72\linewidth}
    \centering
    \includegraphics[width=\linewidth]{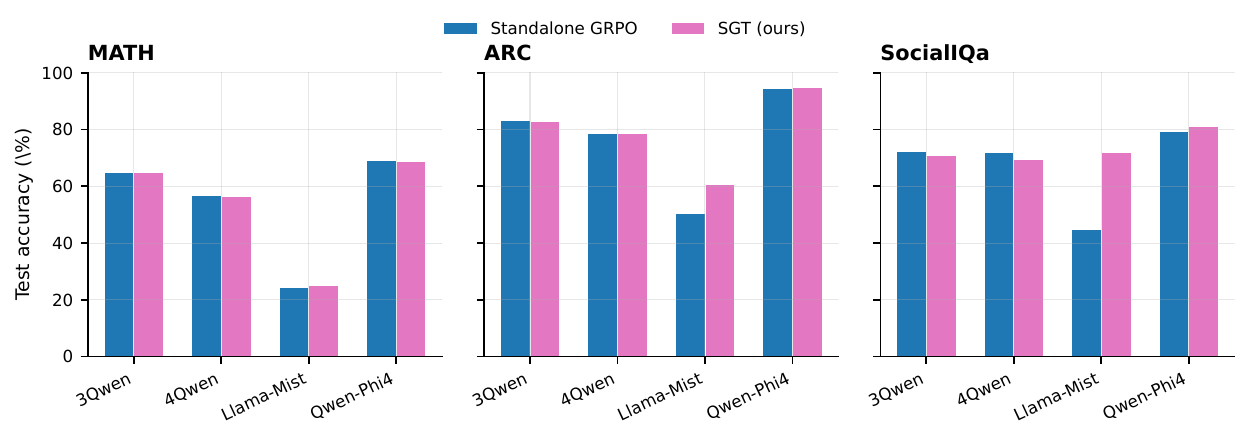}
  \end{minipage}
  \caption{\textbf{Outcome-level transfer extends beyond the main MATH comparison.} Per-task test accuracy on MATH and six commonsense/scientific QA benchmarks for \MFourShort{} compared with Standalone GRPO across four policy pools (3Qwen, 4Qwen, Llama-Mist, Qwen-Phi4). Top row groups the four datasets on which \MFourShort{} adds support across pools (HellaSwag, OpenBookQA, PIQA, BoolQ); bottom row groups the remaining datasets where SGT either matches GRPO at the verifier-grounded MATH task or contributes its largest absolute heterogeneous-pool gains on ARC and SocialIQa. The Llama-Mist pool consistently shows the largest \MFourShort{} gains, consistent with rescue-set transfer between policies that traverse disjoint regions of the solution space.}
  \label{fig:sgt-step50-commonsense}
\end{figure*}

Together with the learning curves in Sec.~\ref{subsec:sgt-q2m-q3b}, these results show that outcome-level transfer propagates verified peer successes across both homogeneous and heterogeneous pools.

\section{Detailed Logits Word Alignment}
\label{app:word-align}
 
This appendix gives the pseudocode for the \THL{} word-level trace-alignment routine that maps source-model per-token log-probabilities onto target-model token positions. The main text (Sec.~\ref{sec:system-design}) defines the mass-preservation rule and the residual bound; this appendix records the concrete span, token-map, and redistribution procedures. The default routine assumes response tokens do not span two distinct whitespace-delimited words, as is typical for BPE/SPM tokenizers that encode spaces with dedicated prefix symbols. When straddling tokens are detected, the implementation apportions values by character overlap or falls back to full-sequence retokenization for that sample.

\subsection{THL WordSpans}
WordSpans creates the tokenizer-agnostic coordinate system used by \THL{}. Direct token-to-token mapping is ill-defined because two tokenizers can segment the same string differently, for example "algorithm" as ["al", "go", "rithm"] versus ["alg", "orithm"]. Algorithm~\ref{alg:thl-word-spans} instead anchors alignment on the raw UTF-8 text.

The procedure scans the string for maximal non-whitespace spans using the regular expression \texttt{\textbackslash S+}. For each match it records the start index and exclusive end index, producing a list of character intervals $(s_k,e_k)$. These intervals are independent of any vocabulary or merge table.

The returned spans define the windows over which source log-probability mass is aggregated and then redistributed onto target tokens. This converts cross-tokenizer alignment from discrete token matching into probability-mass accounting over shared character spans.

\subsection{THL WordAlignLogProbs}
Given WordSpans, WordAlignLogProbs projects a source log-probability trace onto the target tokenizer grid. Algorithm~\ref{alg:thl-word-align} preserves each word's total source log-probability by aggregating over source tokens assigned to that word and redistributing the sum over target tokens assigned to the same word.

The algorithm initializes the output tensor \texttt{out} with an \texttt{ignore\_value}, so padding and masked prompt tokens remain inactive in downstream losses. For each sequence, PerTokenWordMapBySegments links source and target subword tokens to the WordSpans index they cover. Distinct source and target token indices can therefore refer to the same textual word $W_k$.

The transformation has two passes. In the aggregation pass, the algorithm sums source response-token log-probabilities for each word $W_k$:
\[
Z_k=\sum_{t\in W_k}\texttt{source\_log\_probs}[t].
\]
In the redistribution pass, it counts the target tokens $cnt$ assigned to $W_k$ and assigns $Z_k / cnt$ to each of them. The sum over target tokens for the word is then exactly $Z_k$, which is the conservation property used in \Cref{thm:thl-residual}. Uniform redistribution avoids token-count inflation or deflation when source and target tokenizers split the same text into different numbers of subwords.
\begin{algorithm}[t]
\small
\caption{THL: WordSpans (whitespace-delimited spans)}
\label{alg:thl-word-spans}
\begin{algorithmic}[1]
\REQUIRE $text$ (UTF-8 string)
\ENSURE list of spans $[(s_k,e_k)]$ for each maximal non-space run
\STATE spans $\leftarrow [\;]$
\FOR{each regex match $m$ of \texttt{/\textbackslash\textbackslash S+/} in $text$}
  \STATE Append $(m.start,\; m.end)$ to spans
\ENDFOR
\STATE \textbf{return} spans
\end{algorithmic}
\end{algorithm}

\subsection{THL PerTokenWordMapBySegments}
Algorithm~\ref{alg:thl-word-map} links tokenizer-specific subword tokens to the tokenizer-agnostic WordSpans indices. Character offsets alone are insufficient because tokenizers can attach spaces to neighboring tokens or apply context-sensitive merges. The routine therefore reconstructs the tokenization span by span.

For each word span, the algorithm tokenizes the substring from the previous span boundary through the current word end. It appends the resulting token IDs to \texttt{built\_ids} and records the current word index in the parallel \texttt{word\_map}. This creates the token-to-word attribution table consumed by WordAlignLogProbs.

If the text begins with whitespace, the leading delimiter tokens receive \texttt{None} rather than the first word index, so that pure structural delimiter mass is not attributed to the first content word and does not skew the downstream probability alignment.

The final consistency check compares the incrementally reconstructed \texttt{built\_ids} with the full tokenizer output \texttt{full\_ids}, since context-sensitive tokenizer merges can drift between local and global tokenization. When a mismatch occurs, the routine truncates to the common valid prefix, which keeps the map shape-consistent with the token sequence and prevents invalid indexing in the training loop.
\begin{algorithm}[t]
\small
\caption{THL: WordAlignLogProbs (word-level normalization)}
\label{alg:thl-word-align}
\begin{algorithmic}[1]
\REQUIRE texts $[t_i]_{i=1..B}$;\; source\_log\_probs $\in \mathbb{R}^{B\times S_{src}}$;\; source\_response\_mask $\in \{0,1\}^{B\times S_{src}}$
\REQUIRE target\_response\_mask $\in \{0,1\}^{B\times S_{tgt}}$;\; source\_tok;\; target\_tok;\; ignore\_value
\ENSURE out $\in \mathbb{R}^{B\times S_{tgt}}$ with values only where target\_response\_mask==1
\STATE out $\leftarrow$ FullLike(target\_response\_mask, fill=ignore\_value, dtype=source\_log\_probs.dtype)
\FOR{$i \gets 1$ to $B$}
  \STATE $text \leftarrow t_i$
  \STATE $(src\_ids, src\_word\_map) \leftarrow$ PerTokenWordMapBySegments$(text,$ source\_tok$)$
  \STATE $(tgt\_ids, tgt\_word\_map) \leftarrow$ PerTokenWordMapBySegments$(text,$ target\_tok$)$
  \STATE $src\_idx \leftarrow \{j: \text{source\_response\_mask}[i,j]=1\}$
  \STATE $tgt\_idx \leftarrow \{j: \text{target\_response\_mask}[i,j]=1\}$
  \STATE $L_{src} \leftarrow \min(|src\_idx|, |src\_ids|)$;\; $L_{tgt} \leftarrow \min(|tgt\_idx|, |tgt\_ids|)$
  \STATE per\_word\_sum $\leftarrow$ empty map (word\_id $\to$ sum)
  \FOR[sum source log-probs per word]{$p \gets 0$ to $L_{src}-1$}
    \STATE $wid \leftarrow src\_word\_map[p]$; \textbf{continue} if $wid=\text{None}$
    \STATE per\_word\_sum[$wid$] $\leftarrow$ per\_word\_sum.get($wid$,0) $+$ source\_log\_probs[$i,\; src\_idx[p]$]
  \ENDFOR
  \IF{per\_word\_sum is empty} \STATE \textbf{continue} \ENDIF
  \STATE per\_word\_count $\leftarrow$ empty map (word\_id $\to$ count)
  \FOR[count target tokens per word to fill]{$p \gets 0$ to $L_{tgt}-1$}
    \STATE $wid \leftarrow tgt\_word\_map[p]$; \textbf{continue} if $wid=\text{None}$
    \STATE per\_word\_count[$wid$] $\leftarrow$ per\_word\_count.get($wid$,0) $+ 1$
  \ENDFOR
  \FOR[assign normalized values]{$p \gets 0$ to $L_{tgt}-1$}
    \STATE $wid \leftarrow tgt\_word\_map[p]$; \textbf{continue} if $wid=\text{None}$
    \STATE $cnt \leftarrow$ per\_word\_count.get($wid$,0)$;$ \textbf{continue} if $cnt \le 0$
    \STATE $val \leftarrow$ per\_word\_sum.get($wid$,0) $/ cnt$
    \STATE out[$i,\; tgt\_idx[p]$] $\leftarrow val$
  \ENDFOR
\ENDFOR
\STATE \textbf{return} out
\end{algorithmic}
\end{algorithm}

\section{THL alignment diagnostics}
\label{app:thl-diagnostics}

This appendix quantifies three properties of \THL{} beyond the qualitative description in Sec.~\ref{sec:system-design}: (i) alignment error as a function of response length, (ii) the extension to non-whitespace scripts, and (iii) the token-level importance-ratio distribution induced by the \THL{}-aligned peer denominator. All three diagnostics run on frozen models with no parameter updates, isolating \THL{} as a logit-alignment primitive.

\subsection{Long-context alignment error}
\label{subsec:thl-long-context}

We evaluate alignment error as a function of response length on WikiText-103~\citep{merity2016wikitext} spans drawn at four target-token budgets: 1024, 2048, 4096, and 8192 tokens. For each length we sample 64 sequences, giving 256 sequences per model pair. The metric is the relative mean absolute error between the source per-token log-probability tensor and the target-aligned tensor produced by \THL{}, normalized by the mean absolute source word mass on the same span.

The baseline copies source response-token log-probabilities to target response positions at the same indices, truncates excess source tokens, and pads uncovered target positions with zero. Tab.~\ref{tab:thl-long-context} reports the cross-family pair Qwen3-1.7B-Base $\to$ Phi-4-mini-Instruct.

\begin{table}[h]
\centering
\caption{\textbf{Long-context alignment error on heterogeneous Qwen3-1.7B-Base $\to$ Phi-4-mini-Instruct.} Relative mean absolute error (\%) between source per-token log-probabilities and the target-aligned tensor, as a function of target response length. \THL{} stays below 0.03\% across all four lengths, while the position-copy baseline incurs roughly 95--105\% relative error.}
\label{tab:thl-long-context}
\begin{tabular}{rccc}
\toprule
\textbf{Target length} & \textbf{\THL{} rel.\ MAE (\%)} & \textbf{Naive copy/truncate/pad rel.\ MAE (\%)} & \textbf{\THL{} reduction (\%)} \\
\midrule
1024 & 0.0109 & 95.4193 & 99.9876 \\
2048 & 0.0144 & 101.1483 & 99.9859 \\
4096 & 0.0252 & 103.6427 & 99.9758 \\
8192 & 0.0247 & 105.2713 & 99.9768 \\
\bottomrule
\end{tabular}
\end{table}

\THL{} relative error stays in the $10^{-2}$\% range across all four lengths. The position-copy baseline sits at 95--105\% relative error throughout: index-level copying does not approximate the source mass when subword boundaries disagree.

\subsection{Non-whitespace languages}
\label{subsec:thl-cjk}

For non-whitespace scripts such as CJK, the default whitespace anchor degenerates because a single run can span an entire sentence. We replace it with a per-character segmentation when CJK code points are detected, leaving the rest of the alignment pipeline (per-token word map, aggregation, and redistribution) unchanged. Errors are measured per word for Western text and per character for CJK text, matching the unit at which probability mass is conserved.

\begin{table}[h]
\centering
\caption{\textbf{Multilingual alignment error.} Relative MAE (\%) for the original word-level rule, the character-span rule, and an automatic switch that detects CJK code points. The auto-switch reduces CJK error from over 120\% to at most $\sim$1\% on both pools and recovers the original behaviour on Western text.}
\label{tab:thl-cjk}
\begin{tabular}{llccc}
\toprule
\textbf{Pool} & \textbf{Text family} & \makecell{\textbf{Word-level} \\ \textbf{rel.\ MAE (\%)}} & \makecell{\textbf{Char-level} \\ \textbf{rel.\ MAE (\%)}} & \makecell{\textbf{Auto-switch} \\ \textbf{rel.\ MAE (\%)}} \\
\midrule
\multirow{2}{*}{2Qwen} & Western & 0 & 0 & 0 \\
                      & CJK     & 122.2186 & 0 & 0 \\
\midrule
\multirow{2}{*}{Qwen-Phi4} & Western & 0 & 0 & 0 \\
                          & CJK     & 123.0206 & 0.9199 & 0.9199 \\
\bottomrule
\end{tabular}
\end{table}

The original rule is exact on Western (whitespace-delimited) text, exact on the homogeneous 2Qwen pool with character spans, and within $\sim$1\% relative MAE on the cross-family Qwen-Phi4 pool. The CJK auto-switch matches the character-level result on CJK text and the original word-level result on Western text.

\subsection{Importance-ratio distribution under the THL-aligned denominator}
\label{subsec:thl-denominator}

\MOneShort{} (Sec.~\ref{sec:methods}) uses the \THL{}-aligned peer denominator for off-policy importance correction. We measure the resulting token-level ratio distribution against two controls on the same response set: (a) a \emph{shuffled} variant that replaces the peer denominator with one computed on a randomly permuted prompt, and (b) a \emph{broken alignment} variant in which the word-to-token map is intentionally corrupted before redistribution. All three denominators reuse the identical numerator and reward; only the denominator changes. Under the PPO setting $\epsilon=0.2$, ratios outside $[0.8,1.2]$ are clipped.

\begin{table}[h]
\centering
\caption{\textbf{Token-level ratio statistics under the \THL{}-aligned, shuffled, and broken-alignment denominators.} Lower p99, lower clip rate, and a smaller fraction of responses with any token ratio above 10 indicate a tighter, less heavy-tailed ratio distribution. The \THL{}-aligned denominator is the tightest of the three on both pairs.}
\label{tab:thl-denominator-control}
\setlength{\tabcolsep}{1pt}
\resizebox{\linewidth}{!}{%
\begin{tabular}{llccc}
\toprule
\textbf{Source $\to$ Target} & \textbf{Denominator} & \textbf{Token-ratio p99} & \textbf{Token clip rate} & \textbf{Resp.\ with any ratio $>10$} \\
\midrule
\multirow{3}{*}{\makecell[l]{Qwen3-1.7B-Base \\ $\to$ Qwen2.5-Math-1.5B}}
  & \THL{}-aligned peer & 2.1170 & 0.1985 & 0.1021 \\
  & Shuffled prompt     & 4.0806 & 0.3036 & 0.3550 \\
  & Broken alignment    & 4.6240 & 0.2551 & 0.9902 \\
\midrule
\multirow{3}{*}{\makecell[l]{Qwen3-4B-Base \\ $\to$ Phi-4-mini-Instruct}}
  & \THL{}-aligned peer & 2.1902 & 0.3214 & 0.1982 \\
  & Shuffled prompt     & 4.4121 & 0.4202 & 0.4136 \\
  & Broken alignment    & 4.5087 & 0.3585 & 0.9868 \\
\bottomrule
\end{tabular}%
}
\end{table}

On both pairs, the \THL{}-aligned denominator yields a smaller token-ratio tail than either control. Under broken alignment, almost every response contains at least one token ratio above 10; under shuffled prompts, more than a third do. The \THL{}-aligned denominator drops this fraction to 10--20\% and roughly halves the p99 ratio.

At the sequence level (Tab.~\ref{tab:thl-seq-ratio}), 1{,}280 responses per source--target pair give finite, bounded ratios on both pairs: the same-family pair has median 0.9104 and p99 0.9834, the cross-family pair has median 0.7466 and p99 0.9057.

\begin{table}[h]
\centering
\caption{\textbf{\THL{} yields finite sequence-level importance ratios on 1{,}280 sampled responses per pair.} The same-family Qwen pair stays inside the PPO clip band for 89.06\% of responses, while the cross-family Qwen-to-Phi4 pair remains finite but clips more often, matching the stronger tokenizer and model mismatch.}
\label{tab:thl-seq-ratio}
\begin{tabular}{lccc}
\toprule
\textbf{Source $\to$ Target} & \textbf{Median} & \textbf{p99} & \textbf{Clip rate ($[0.8,1.2]$)} \\
\midrule
Qwen3-1.7B-Base $\to$ Qwen2.5-Math-1.5B & 0.9104 & 0.9834 & 0.1094 \\
Qwen3-4B-Base $\to$ Phi-4-mini-Instruct & 0.7466 & 0.9057 & 0.6969 \\
\bottomrule
\end{tabular}
\end{table}

\section{Pooled normalization vs.\ shuffled-pool control}
\label{app:xgrpo-pooled}

\MThreeShort{} replaces the per-policy reward baseline with a baseline pooled across policies on the same prompt. We measure how much of the resulting advantage adjustment is driven by the cross-policy reward structure rather than by random replacement of the pool.

\paragraph{Setup.} For every prompt in a shared rollout batch, we compute (a) the standard per-policy baseline, (b) the true pooled baseline used by \MThreeShort{}, and (c) a \emph{shuffled-pool} baseline in which the peer reward bag is randomly permuted across prompts before pooling. The numerator and the policy gradient are otherwise identical. We measure three quantities. The \emph{correlation} is between the per-prompt peer reward spread and the change in learner advantage. The \emph{advantage sign-flip rate} is the fraction of prompts on which the sign of the learner's advantage changes relative to the per-policy baseline. The \emph{mean absolute advantage change} is the average magnitude of that change.

\begin{table}[h]
\centering
\small
\caption{\textbf{Pooled normalization carries structured cross-policy reward information.} On the same shared batches, the true cross-policy pool produces stronger correlation between peer reward spread and learner advantage adjustment, while keeping both the sign-flip rate and the magnitude of the change smaller than the shuffled-pool control.}
\label{tab:xgrpo-pool-control}
\setlength{\tabcolsep}{1pt}
\resizebox{\linewidth}{!}{%
\begin{tabular}{llcccccc}
\toprule
\textbf{Pool} & \textbf{Learner} & \multicolumn{2}{c}{\textbf{Corr.\ peer spread $\leftrightarrow$ adv.\ change}} & \multicolumn{2}{c}{\textbf{Adv.\ sign-flip rate}} & \multicolumn{2}{c}{\textbf{Mean abs.\ adv.\ change}} \\
& & True & Shuffled & True & Shuffled & True & Shuffled \\
\midrule
\multirow{2}{*}{2Qwen} & Qwen2.5-Math-1.5B & 0.2863 & $-$0.0526 & 0.0836 & 0.2906 & 0.2461 & 0.4641 \\
                       & Qwen3-1.7B-Base   & 0.4108 & $-$0.0672 & 0.0984 & 0.3094 & 0.2418 & 0.4782 \\
\midrule
\multirow{2}{*}{Qwen-Phi4} & Phi-4-mini-Instruct & $-$0.0235 & $-$0.3293 & 0.3117 & 0.4539 & 0.5514 & 0.6783 \\
                           & Qwen3-4B-Base       &  0.1592 & $-$0.1136 & 0.1562 & 0.3281 & 0.5232 & 0.6630 \\
\bottomrule
\end{tabular}%
}
\end{table}

The true pool yields stronger coupling with peer reward spread than the shuffled control: averaged over the four learners, mean correlation is $0.2082$ for the true pool and $-0.1407$ for the shuffled pool. The true pool also flips advantage signs less often and produces smaller absolute changes. \MThreeShort{} therefore reshapes credit through coherent cross-policy reward statistics rather than through generic baseline noise. The pattern holds across the four learners, which span the Qwen and Phi-4 model families, consistent with the stable accuracy gains in Sec.~\ref{subsec:sgt-q2m-q3b} and App.~\ref{subsec:xgrpo-ablation}.

\section{SGT diagnostics: cost, activation profile, and prompt specificity}
\label{app:sgt-diagnostics}

This appendix complements the implementation description in App.~\ref{app:sgt-details} with three measurements: (i) extra compute consumed by each of the three sharing regimes under a fixed rollout budget, (ii) the per-prompt activation profile of \MFourShort{}, and (iii) the prompt-specificity of the \MFourShort{} signal via a matched- vs.\ mismatched-teacher comparison.

\subsection{Compute overhead under a fixed rollout budget}
\label{subsec:sgt-cost}

We measure compute on a shared MATH rollout batch drawn from both the 2Qwen and the Qwen-Phi4 settings, with 256 prompts per pool and 3{,}238{,}926 response tokens in total, plus the corresponding teacher traces consumed by \MFourShort{}. \MThreeShort{} requires no extra model work: it only changes how rewards are normalized when computing advantages. \MFourShort{} adds a sparse auxiliary supervised loss on a small set of selected peer responses, contributing 21{,}168 additional teacher tokens in the 2Qwen run. \MOneShort{} re-scores each pooled rollout under the learner's policy via the \THL{}-aligned denominator, costing another full pass of cross-policy compute.

\begin{table}[h]
\centering
\caption{\textbf{Extra compute on the same rollout batch.} \MThreeShort{} adds no extra model work after the batch is collected. \MFourShort{} adds a small fraction of one rollout. \MOneShort{} adds approximately one extra rollout because every peer trajectory is re-scored.}
\label{tab:regime-cost}
\begin{tabular}{lcc}
\toprule
\textbf{Regime} & \textbf{Extra response tokens} & \textbf{Extra cost} \\
\midrule
\MThreeShort{} & 0           & $0\times$ rollout \\
\MFourShort{}  & 21{,}168    & $0.0065\times$ rollout \\
\MOneShort{}   & 3{,}238{,}926 & $1\times$ rollout \\
\bottomrule
\end{tabular}
\end{table}

The auxiliary cost of \MFourShort{} has an analytic upper bound. With $M$ models each generating $N$ rollouts, every prompt that triggers \MFourShort{} for at least one learner contains at least one peer success that is not itself triggered. The worst-case extra sequence count is therefore bounded by $(M-1)/(MN)$ of the rollout batch, matching the empirical $\sim$0.65\% in Tab.~\ref{tab:regime-cost}.

\subsection{Per-prompt activation profile}
\label{subsec:sgt-where-fires}

\MFourShort{} activates on prompts where the learner fails and a peer succeeds. We instrument the gate over a 256-prompt MATH batch in the 2Qwen and Qwen-Phi4 settings, and track for each learner: how often the gate fires, and the per-prompt success rate on the gated and ungated subsets.

\begin{table}[h]
\centering
\caption{\textbf{\MFourShort{} activation profile.} The gated subset has lower pool-wide success rates (12/90 to 18/70 successful rollouts) than the ungated subset, where 1263/2490 to 1427/2500 rollouts succeed. \MFourShort{} therefore activates on the hard subset of the pool's prompt distribution.}
\label{tab:sgt-selectivity}
\setlength{\tabcolsep}{4pt}
\begin{tabular}{llcccc}
\toprule
\textbf{Pool} & \textbf{Learner} & \textbf{SGT prompts} & \makecell{\textbf{Succ./all} \\ \textbf{(gate on)}} & \makecell{\textbf{Succ./all} \\ \textbf{(gate off)}} & \textbf{All-fail prompts} \\
\midrule
\multirow{2}{*}{2Qwen} & Qwen2.5-Math-1.5B & 9/256 & 12/90 & 1269/2470 & \multirow{2}{*}{61/256} \\
                       & Qwen3-1.7B-Base   & 7/256 & 18/70 & 1263/2490 & \\
\midrule
\multirow{2}{*}{Qwen-Phi4} & Phi-4-mini-Instruct & 7/256 & 18/70 & 1420/2490 & \multirow{2}{*}{56/256} \\
                           & Qwen3-4B-Base       & 6/256 & 11/60 & 1427/2500 & \\
\bottomrule
\end{tabular}
\end{table}

The activation set is the rescue subset: prompts where one model fails and a peer solves the same question. The pool-wide success rate on this subset is roughly $1/7$ of the success rate on the rest of the batch, identifying it as the hard tail.

The activation density scales with task difficulty. On the 3Qwen pool, the gate fires on 8.98\% of MATH prompts, 24.35\% of Social IQa prompts, and 31.38\% of HellaSwag prompts; tasks with denser per-prompt reasoning paths and higher absolute success rates produce more rescue prompts.

\subsection{Matched vs.\ mismatched teacher: prompt-specificity of the SGT signal}
\label{subsec:sgt-matched-teacher}

\MFourShort{} transfers a successful peer trace on a prompt where the learner fails. We measure whether the transferred signal is prompt-specific by computing the token-normalized NLL of the matched teacher (the trace \MFourShort{} would actually use) against a mismatched teacher (a successful trace from the same peer on a \emph{different} prompt). Both teachers are successful traces from the same peer; only the prompt-level alignment differs.

\begin{table}[h]
\centering
\caption{\textbf{Matched vs.\ mismatched teacher NLL on the SGT-triggered subset.} The learner consistently assigns lower NLL to the prompt-matched teacher than to a mismatched one across all four learners drawn from the Qwen and Phi-4 model families, with the largest gap on the cross-family Qwen3-4B-Base learner.}
\label{tab:sgt-matched}
\begin{tabular}{llcc}
\toprule
\textbf{Pool} & \textbf{Learner} & \textbf{Matched teacher NLL} & \textbf{Mismatched teacher NLL} \\
\midrule
\multirow{2}{*}{2Qwen} & Qwen2.5-Math-1.5B & 0.3243 & 0.4207 \\
                       & Qwen3-1.7B-Base   & 0.2347 & 0.3461 \\
\midrule
\multirow{2}{*}{Qwen-Phi4} & Phi-4-mini-Instruct & 0.5483 & 0.5898 \\
                           & Qwen3-4B-Base       & 0.7975 & 1.0511 \\
\bottomrule
\end{tabular}
\end{table}

Across all four learners the matched teacher attains a strictly lower NLL than the mismatched one, with the largest gap on the cross-family Qwen3-4B-Base learner (0.7975 vs.\ 1.0511). Both teachers are drawn from the same peer model and are both successful, so the gap isolates the prompt-specific component of the signal that \MFourShort{} transfers.

\section{Cross-policy complementarity and channel decomposition}
\label{app:complementarity}

This appendix quantifies the prompt-level signal that mutual RL makes available: peers solve a measurable fraction of the prompts a learner misses, this complementarity concentrates on harder prompts, and the three sharing regimes expose different subsets of the same shared batch.

\subsection{Prompt-level complementarity of the policy pool}
\label{subsec:complementarity-prompts}

We sample 500 MATH prompts and record per-model success indicators under fixed decoding for the 2Qwen and Qwen-Phi4 pools, which together span the Qwen and Phi-4 model families. From these we compute (a) per-model success rates, (b) Jaccard overlap between any two models' success sets, (c) the conditional probability that a peer succeeds given that the learner fails, and (d) pool-level statistics: the fraction of prompts solved by \emph{any} model, by \emph{all} models, and by \emph{exactly one} model.

\begin{table}[h]
\centering
\caption{\textbf{Prompt-level complementarity on 500 MATH prompts.} The peer rescues the learner with conditional probability 0.21--0.26 across the four learner-peer pairings. Per-model success rates sit at 0.65--0.67; the any-model pool success is 0.74.}
\label{tab:complementarity-pairs}
\setlength{\tabcolsep}{4pt}
\begin{tabular}{llccccc}
\toprule
\textbf{Pool} & \textbf{Learner} & \textbf{Peer} & \makecell{\textbf{Learner} \\ \textbf{succ.}} & \makecell{\textbf{Peer} \\ \textbf{succ.}} & \textbf{Jaccard} & \makecell{\textbf{Peer rescues} \\ \textbf{fail}} \\
\midrule
\multirow{2}{*}{2Qwen} & Qwen2.5-Math-1.5B & Qwen3-1.7B-Base   & 0.6480 & 0.6540 & \multirow{2}{*}{0.7642} & 0.2557 \\
                       & Qwen3-1.7B-Base   & Qwen2.5-Math-1.5B & 0.6540 & 0.6480 &                          & 0.2428 \\
\midrule
\multirow{2}{*}{Qwen-Phi4} & Phi-4-mini-Instruct & Qwen3-4B-Base       & 0.6700 & 0.6580 & \multirow{2}{*}{0.7946} & 0.2121 \\
                           & Qwen3-4B-Base       & Phi-4-mini-Instruct & 0.6580 & 0.6700 &                          & 0.2398 \\
\bottomrule
\end{tabular}
\end{table}

\begin{table}[h]
\centering
\caption{\textbf{Pool-level success on the same 500 MATH prompts.} Any-model success exceeds mean single-model success on both pools; exactly-one-model success reaches 17.4\% on 2Qwen and 15.2\% on Qwen-Phi4.}
\label{tab:complementarity-pool}
\begin{tabular}{lcccc}
\toprule
\textbf{Pool} & \textbf{Mean single-model} & \textbf{Any-model} & \textbf{All-model} & \textbf{Exactly-one-model} \\
\midrule
2Qwen     & 0.6510 & 0.7380 & 0.5640 & 0.1740 \\
Qwen-Phi4 & 0.6640 & 0.7400 & 0.5880 & 0.1520 \\
\bottomrule
\end{tabular}
\end{table}

Per-model accuracy is 0.65--0.67; the any-model pool clears 0.73--0.74. The 8--9 percentage-point gap is the prompt set on which one model succeeds and its peer does not, and the conditional rescue rate of 0.21--0.26 quantifies the fraction of learner failures that a single peer covers. \MFourShort{} converts this gap into supervised signal; \MOneShort{} converts it into off-policy gradient updates.

\subsection{Difficulty-conditioned diversity}
\label{subsec:complementarity-difficulty}

We stratify the same 500 MATH prompts into equal-sized easy, medium, and hard buckets by mean verifier reward, and recompute Jaccard overlap and exactly-one-model success within each bucket.

\begin{table}[h]
\centering
\caption{\textbf{Difficulty-conditioned diversity on 500 MATH prompts, equal-sized easy/medium/hard buckets.} Jaccard overlap is near-perfect on easy prompts and drops to 0.48--0.51 on hard prompts in both pools, with exactly-one-model success rising to 15--16\%.}
\label{tab:complementarity-difficulty}
\setlength{\tabcolsep}{6pt}
\begin{tabular}{lcccc}
\toprule
\textbf{Bucket} & \makecell{\textbf{2Qwen} \\ \textbf{Jaccard}} & \makecell{\textbf{2Qwen} \\ \textbf{one-model}} & \makecell{\textbf{Qwen-Phi4} \\ \textbf{Jaccard}} & \makecell{\textbf{Qwen-Phi4} \\ \textbf{one-model}} \\
\midrule
Easy   & 1.0000 & 0      & 1.0000 & 0      \\
Medium & 0.9701 & 0.0299 & 0.9940 & 0.0060 \\
Hard   & 0.4792 & 0.1506 & 0.5091 & 0.1627 \\
\bottomrule
\end{tabular}
\end{table}

The aggregate diversity in Tab.~\ref{tab:complementarity-pool} concentrates on the hard bucket. On easy prompts the two models agree perfectly. On the hard bucket Jaccard overlap drops to 0.48--0.51 and exactly-one-model success rises to 0.15--0.16. This matches the per-prompt activation profile of \MFourShort{} reported in App.~\ref{subsec:sgt-where-fires}: the gate fires almost exclusively on prompts where pool-wide success is low.

\subsection{Channel decomposition of the three sharing regimes}
\label{subsec:regime-channels}

\MOneShort{}, \MThreeShort{}, and \MFourShort{} isolate three coupling levels. \MOneShort{} exposes the trajectory channel: a learner optimizes directly on peer rollouts. \MThreeShort{} exposes the scalar-value channel: learner trajectories are unchanged and only pooled reward statistics enter the advantage. \MFourShort{} exposes the outcome channel: verified peer successes are transferred only on learner-failure prompts. The three channels are complementary rather than nested, because rollout usability, scalar-baseline movement, and rescue events occur on different prompt subsets.

We measure their joint activation pattern on a single shared batch. A prompt is \emph{usable} for \MOneShort{} if at least one peer response has importance ratio inside $[0.8, 1.2]$ for the learner; usable for \MThreeShort{} if pooled normalization changes the learner's advantage relative to per-model normalization; usable for \MFourShort{} if the learner fails and a peer succeeds. \MFourShort{} usability implies \MThreeShort{} usability, since a peer-success / learner-failure mismatch changes the pooled reward statistics.

\begin{table}[h]
\centering
\caption{\textbf{Joint usability of the three regimes on the same shared batch.} The dominant case is \MOneShort{} and \MThreeShort{} both usable but \MFourShort{} not. \MOneShort{}-only prompts cover 19.43\% of the batch, \MThreeShort{}-only prompts cover 14.45\%, and \MFourShort{} is usable on 2.64\% of prompts.}
\label{tab:regime-overlap}
\begin{tabular}{lc}
\toprule
\textbf{Usable regimes on prompt} & \textbf{\% of prompts} \\
\midrule
None of the three                              &  4.39 \\
Only \MThreeShort{}                            & 14.45 \\
\MThreeShort{} and \MFourShort{}, not \MOneShort{} &  0.59 \\
Only \MOneShort{}                              & 19.43 \\
\MOneShort{} and \MThreeShort{}, not \MFourShort{} & 59.08 \\
All three                                       &  2.05 \\
\bottomrule
\end{tabular}
\end{table}

The overlap pattern matches the stability-support interpretation. \MOneShort{} reaches the broadest set of peer trajectories but pays the density-ratio cost; \MThreeShort{} supplies dense scalar shaping on most prompts; \MFourShort{} fires sparsely on the rescue subset where a verified peer success supplies trajectory support the learner did not sample.

\section{Theoretical analysis of the three communication channels}
\label{app:theory}

\subsection{Verifiable contextual-bandit formulation}
\label{subsec:theory-setup}

We work in the same prompt-level, critic-free setting as the GRPO objective in Sec.~\ref{sec:preliminaries}.  A prompt $x$ is drawn from $\mathcal{D}$, a policy samples a complete response $y$, and a verifier returns a binary outcome $r(x,y)\in\{0,1\}$.  The success set is
\[
\mathcal{S}_x=\{y:r(x,y)=1\}.
\]
For compactness write $\pi_n(\cdot\mid x)=\pi^{(n)}_\theta(\cdot\mid x)$ when the current parameter is clear.  The per-policy success probability and rollout group are
\[
p_n(x)=\Pr_{y\sim\pi_n(\cdot\mid x)}[y\in\mathcal{S}_x],
\qquad
Y_n(x)=\{y_{n,1},\ldots,y_{n,K}\},
\qquad
y_{n,j}\overset{\mathrm{i.i.d.}}{\sim}\pi_n(\cdot\mid x).
\]
The successful self-rollout set is $\mathcal{S}^{(n)}_x=Y_n(x)\cap\mathcal{S}_x$.  Under independent $K$-rollout groups across policies, the probability that at least one peer of learner $n$ succeeds is
\[
q_{-n}(x)
=
1-\prod_{m\neq n}(1-p_m(x))^K.
\]
For a token-level log-probability trace $\ell^{(n)}_\theta(x,y)$, we write
\[
L^{(n)}_\theta(x,y)=\sum_t \ell^{(n)}_{\theta,t}(x,y)
\]
for its response-level sum.  The usual score-function identity is used only at this single-step response level, as in REINFORCE/PPO/GRPO-style objectives~\citep{Williams1992REINFORCE,Schulman2017PPO,shao2024deepseekmath}.

The three channels expose different functionals of the same prompt-level sample.

\paragraph{\MOneShort{}: data-level coupling.}
\MOneShort{} consumes peer trajectories $y\sim\mu^{(m)}(\cdot\mid x)$ and inserts them into learner $n$'s policy-gradient surrogate.  The learner evaluates its own trace $\ell^{(n)}_\theta(x,y)$ on the retokenized peer text and forms a sequence-level ratio
\[
w^{(n)}_\theta(x,y)
=
\exp\{L^{(n)}_\theta(x,y)-L_{\mathrm{behavior}}(x,y)\}.
\]
The naive variant uses the learner snapshot denominator $L_{\mathrm{behavior}}=L^{(n)}_{\theta_{\mathrm{old}}}$ even on peer data.  The corrected variant uses the peer behavior denominator, represented on learner $n$'s token grid by $\tilde\ell^{(m\to n)}$ and summed into $\tilde L^{(m\to n)}(x,y)$.  This is the high-coupling channel: it transfers complete trajectories and therefore must pay density-ratio and alignment-error costs.

\paragraph{\MThreeShort{}: value-level coupling.}
\MThreeShort{} consumes only scalar rewards.  For prompt $x$, learner $n$ samples its own $y_{n,i}\sim\pi_n(\cdot\mid x)$ but computes advantages using pooled reward statistics, for example
\[
\mu_{\mathrm{pool}}(x)=\frac{1}{MK}\sum_{m=1}^M\sum_{j=1}^K r(x,y_{m,j}),
\qquad
A^{(n)}_i=\frac{r(x,y_{n,i})-\mu_{\mathrm{pool}}(x)}
{\sigma_{\mathrm{pool}}(x)+\epsilon}.
\]
The actor score term remains $\nabla_\theta\log\pi^{(n)}_\theta(y_{n,i}\mid x)$ for learner-sampled responses only.

\paragraph{\MFourShort{}: outcome-level coupling.}
\MFourShort{} consumes a verified peer success only on the rescue subset.  The gate fires on prompt $x$ when
\[
Y_n(x)\cap\mathcal{S}_x=\varnothing
\quad\text{and}\quad
\mathcal{S}_x\cap\bigcup_{m\neq n}Y_m(x)\neq\varnothing.
\]
On this event, learner $n$ receives a selected $y^*\in\mathcal{S}_x\cap\bigcup_{m\neq n}Y_m(x)$ and adds the auxiliary loss
\[
\mathcal{L}^{(n)}_{\mathrm{SGT}}(x,y^*)=-\log\pi^{(n)}_\theta(y^*\mid x)
\]
with weight $\lambda$ or $\lambda_{\mathrm{SGT}}$.  If the gate is off, the learner's update is the base GRPO update.

\subsection{\THL{} alignment: conservation, residual error, and ratio envelopes}
\label{subsec:theory-thl}

\begin{assumption}[\THL{} span and clipping convention]
\label{ass:thl-span}
Fix a response text $y$.  WordSpans returns a partition into units indexed by $w$; for whitespace-delimited text these are maximal non-whitespace spans, with the delimiter convention used by PerTokenWordMapBySegments.  Source and target token spans are mapped to word indices by PerTokenWordMapBySegments.  A response token is \emph{valid} if it is assigned to a single word index and is included by the response mask; otherwise it is boundary-mismatched or ignored.  WordAlignLogProbs aggregates valid source log-probabilities by word and redistributes each word sum uniformly over valid target tokens assigned to that word.  All response log-probabilities used in the denominator are clipped or masked so that $|\ell^{\mathrm{src}}_t|\leq B$.
\end{assumption}

For a word $w$, let
\[
S_w=\{t:\ a_{\mathrm{src}}(t)=w,\ t\text{ valid}\},
\qquad
T_w=\{u:\ a_{\mathrm{tgt}}(u)=w,\ u\text{ valid}\},
\qquad
C_w=|T_w|.
\]
Let
\[
Z_w^{\mathrm{src}}=\sum_{t\in S_w}\ell^{\mathrm{src}}_t.
\]
Let $\mathcal{B}_{\mathrm{src}}$ be the set of source response tokens that are not assigned to any word by the valid mapping, and let
\[
\mathcal{W}_0=\{w:\ S_w\neq\varnothing,\ C_w=0\}
\]
be the set of source-covered words with no valid target slot.  Define the \THL{} residual
\[
R_{\mathrm{THL}}(x,y)
=
\sum_{t\in\mathcal{B}_{\mathrm{src}}}\ell^{\mathrm{src}}_t
+
\sum_{w\in\mathcal{W}_0}Z_w^{\mathrm{src}},
\]
the count
\[
C_{\mathrm{mis}}(x,y)
=
|\mathcal{B}_{\mathrm{src}}|
+
\sum_{w\in\mathcal{W}_0}|S_w|,
\]
and the sharper residual magnitude
\[
\delta_{\mathrm{THL}}(x,y)=|R_{\mathrm{THL}}(x,y)|.
\]

\begin{theorem}[\THL{} conservation with a derived residual bound]
\label{thm:thl-residual}
Assume Assumption~\ref{ass:thl-span}.  Let
\[
L_\mu(x,y)
=
\sum_{w}\sum_{t\in S_w}\ell^{\mathrm{src}}_t
+
\sum_{t\in\mathcal{B}_{\mathrm{src}}}\ell^{\mathrm{src}}_t
\]
be the masked source sequence log-probability under the peer tokenizer, and let
\[
\tilde L_\mu(x,y)=\sum_{u:\ a_{\mathrm{tgt}}(u)\neq\bot}\tilde\ell^{(\mathrm{src}\to\mathrm{tgt})}_u
\]
be the target-grid denominator produced by WordAlignLogProbs, excluding ignore-value positions.  Then
\[
L_\mu(x,y)-\tilde L_\mu(x,y)=R_{\mathrm{THL}}(x,y),
\]
and therefore
\[
|L_\mu(x,y)-\tilde L_\mu(x,y)|
=
\delta_{\mathrm{THL}}(x,y)
\leq
\sum_{t\in\mathcal{B}_{\mathrm{src}}}|\ell^{\mathrm{src}}_t|
+
\sum_{w\in\mathcal{W}_0}\sum_{t\in S_w}|\ell^{\mathrm{src}}_t|
\leq
B\,C_{\mathrm{mis}}(x,y).
\]
If the source and target segmentations are refinements of the same word partition and every source-covered word has at least one valid target token, then $C_{\mathrm{mis}}(x,y)=0$ and \THL{} preserves the sequence denominator exactly.  In particular, for each word with $C_w>0$,
\[
\sum_{u\in T_w}\tilde\ell^{(\mathrm{src}\to\mathrm{tgt})}_u
=
\sum_{t\in S_w}\ell^{\mathrm{src}}_t .
\]
\end{theorem}

\begin{proof}
By the aggregation pass of WordAlignLogProbs, the value stored for word $w$ is $Z_w^{\mathrm{src}}$.  If $C_w>0$, the redistribution pass assigns
\[
\tilde\ell^{(\mathrm{src}\to\mathrm{tgt})}_u
=
\frac{Z_w^{\mathrm{src}}}{C_w}
\qquad
\text{for each }u\in T_w.
\]
Hence the total aligned mass placed on target tokens of word $w$ is
\[
\sum_{u\in T_w}\tilde\ell^{(\mathrm{src}\to\mathrm{tgt})}_u
=
\sum_{u\in T_w}\frac{Z_w^{\mathrm{src}}}{C_w}
=
Z_w^{\mathrm{src}}.
\]
If $C_w=0$, no target token receives the source word mass $Z_w^{\mathrm{src}}$.  Summing over all words gives
\[
\tilde L_\mu(x,y)
=
\sum_{w:C_w>0}Z_w^{\mathrm{src}}.
\]
The source sequence sum decomposes as
\[
L_\mu(x,y)
=
\sum_{w:C_w>0}Z_w^{\mathrm{src}}
+
\sum_{w\in\mathcal{W}_0}Z_w^{\mathrm{src}}
+
\sum_{t\in\mathcal{B}_{\mathrm{src}}}\ell^{\mathrm{src}}_t.
\]
Subtracting the previous two displays yields
$L_\mu(x,y)-\tilde L_\mu(x,y)=R_{\mathrm{THL}}(x,y)$.  Taking absolute values and using the triangle inequality gives the sharper mass bound.  The final inequality follows from $|\ell^{\mathrm{src}}_t|\leq B$.  Under exact refinement, $\mathcal{B}_{\mathrm{src}}=\varnothing$ and $\mathcal{W}_0=\varnothing$, so the residual is zero and the per-word equality above holds for every word.
\end{proof}

\begin{corollary}[Importance-ratio envelope induced by \THL{}]
\label{cor:thl-ratio}
Assume the conditions of \Cref{thm:thl-residual}.  Let
\[
\rho(x,y)
=
\exp\{L^{(n)}_\theta(x,y)-L_\mu(x,y)\},
\qquad
\tilde\rho(x,y)
=
\exp\{L^{(n)}_\theta(x,y)-\tilde L_\mu(x,y)\}.
\]
Then
\[
e^{-\delta_{\mathrm{THL}}(x,y)}\rho(x,y)
\leq
\tilde\rho(x,y)
\leq
e^{\delta_{\mathrm{THL}}(x,y)}\rho(x,y),
\qquad
\delta_{\mathrm{THL}}(x,y)\leq B\,C_{\mathrm{mis}}(x,y).
\]
\end{corollary}

\begin{proof}
The ratio of the aligned and ideal weights is
\[
\frac{\tilde\rho(x,y)}{\rho(x,y)}
=
\exp\{L_\mu(x,y)-\tilde L_\mu(x,y)\}
=
\exp\{R_{\mathrm{THL}}(x,y)\}.
\]
Since $|R_{\mathrm{THL}}(x,y)|=\delta_{\mathrm{THL}}(x,y)$, the exponential lies in
$[e^{-\delta_{\mathrm{THL}}(x,y)},e^{\delta_{\mathrm{THL}}(x,y)}]$.  Multiplying by $\rho(x,y)>0$ proves the envelope.  The count-based bound is exactly the final inequality of \Cref{thm:thl-residual}.
\end{proof}

The corollary is the THL alignment guarantee used by \MOneShort{}. THL preserves the peer denominator exactly when source and target tokenizations refine the same word partition; under segmentation mismatch, the only remaining denominator error is the explicit residual from mismatched or uncovered spans. The residual enters the PRP importance ratio multiplicatively, so THL should be read as a trace-alignment primitive with a controlled envelope, not as an exact probability-correction oracle.

\subsection{Data-level sharing (\MOneShort{}): density-ratio variance and alignment error}
\label{subsec:theory-prp}

\begin{assumption}[\MOneShort{} estimator family]
\label{ass:prp}
Fix a prompt $x$, learner $n$, and peer $m$.  Write
\[
\pi(\cdot)=\pi^{(n)}_\theta(\cdot\mid x),
\qquad
\mu(\cdot)=\mu^{(m)}(\cdot\mid x),
\]
and assume $\pi\ll\mu$.  Let
\[
s_\theta(y)=\nabla_\theta\log\pi^{(n)}_\theta(y\mid x),
\qquad
h(y)=A(x,y)s_\theta(y),
\qquad
\rho(y)=\frac{\pi(y)}{\mu(y)}.
\]
Assume $|A(x,y)|\leq A_{\max}$ and $\|s_\theta(y)\|\leq G$ for all $y$ with $\mu(y)>0$.  Let $\tilde\rho(y)$ be the \THL{}-aligned ratio and suppose it satisfies the envelope in \Cref{cor:thl-ratio} with residual $\delta_{\mathrm{THL}}(x,y)\leq\bar\delta_{\mathrm{THL}}$.  Let
\[
c_\epsilon(z)=\operatorname{clip}(z,1-\epsilon,1+\epsilon).
\]
\end{assumption}

\begin{theorem}[\MOneShort{} bias--variance decomposition]
\label{thm:prp-bv}
Under Assumption~\ref{ass:prp}, define
\[
g_\pi(x)=\mathbb{E}_{y\sim\pi}[h(y)],
\qquad
\hat g_{\mathrm{IS}}=\rho(y)h(y),\ y\sim\mu,
\]
\[
g_{\mathrm{clip}}(x)=\mathbb{E}_{\mu}[c_\epsilon(\rho(y))h(y)],
\qquad
g_{\mathrm{clip,THL}}(x)=\mathbb{E}_{\mu}[c_\epsilon(\tilde\rho(y))h(y)].
\]
Then:
\[
\mathbb{E}_{\mu}[\hat g_{\mathrm{IS}}]=g_\pi(x),
\]
and, using mean-square vector variance,
\[
\operatorname{Var}(\hat g_{\mathrm{IS}})
\leq
A_{\max}^2G^2\left(1+\chi^2(\pi\,\|\,\mu)\right),
\qquad
\chi^2(\pi\,\|\,\mu)=\mathbb{E}_{\mu}[\rho(y)^2]-1.
\]
The exact-ratio clipping bias is
\[
B_{\mathrm{clip}}(\epsilon)
=
\left\|
\mathbb{E}_{\mu}[(c_\epsilon(\rho(y))-\rho(y))h(y)]
\right\|
\]
and satisfies the tail bound
\[
B_{\mathrm{clip}}(\epsilon)
\leq
A_{\max}G\,
\mathbb{E}_{\mu}\!\left[
(\rho(y)-(1+\epsilon))_+
+
((1-\epsilon)-\rho(y))_+
\right].
\]
The clipped estimator has bounded second moment
\[
\mathbb{E}_{\mu}\|c_\epsilon(\rho(y))h(y)\|^2
\leq
(1+\epsilon)^2A_{\max}^2G^2.
\]
Finally, THL alignment changes the clipped expectation by at most
\[
\|g_{\mathrm{clip,THL}}(x)-g_{\mathrm{clip}}(x)\|
\leq
A_{\max}G\left(e^{\bar\delta_{\mathrm{THL}}}-1\right),
\]
and the un-clipped THL-aligned second moment obeys
\[
\mathbb{E}_{\mu}\|\tilde\rho(y)h(y)\|^2
\leq
e^{2\bar\delta_{\mathrm{THL}}}
A_{\max}^2G^2\left(1+\chi^2(\pi\,\|\,\mu)\right).
\]
\end{theorem}

\begin{proof}
The exact importance-weighted estimator is unbiased because
\[
\mathbb{E}_{\mu}[\rho(y)h(y)]
=
\sum_y \mu(y)\frac{\pi(y)}{\mu(y)}h(y)
=
\sum_y \pi(y)h(y)
=
g_\pi(x).
\]
For the variance,
\[
\operatorname{Var}(\hat g_{\mathrm{IS}})
=
\mathbb{E}_{\mu}\|\rho h-g_\pi\|^2
\leq
\mathbb{E}_{\mu}\|\rho h\|^2
\leq
A_{\max}^2G^2\mathbb{E}_{\mu}[\rho^2].
\]
By definition, $\mathbb{E}_{\mu}[\rho^2]=1+\chi^2(\pi\|\mu)$, which yields the displayed bound.  For clipping,
\[
g_{\mathrm{clip}}-g_\pi
=
\mathbb{E}_{\mu}[(c_\epsilon(\rho)-\rho)h],
\]
so the first clipping display is exact.  Since $c_\epsilon(\rho)-\rho$ is nonzero only in the two ratio tails,
\[
|c_\epsilon(\rho)-\rho|
\leq
(\rho-(1+\epsilon))_+
+
((1-\epsilon)-\rho)_+,
\]
and $\|h\|\leq A_{\max}G$ gives the tail bound.  The second-moment cap follows from
$|c_\epsilon(\rho)|\leq 1+\epsilon$:
\[
\mathbb{E}_{\mu}\|c_\epsilon(\rho)h\|^2
\leq
(1+\epsilon)^2A_{\max}^2G^2.
\]
For the THL term, $c_\epsilon$ is $1$-Lipschitz, so
\[
\|g_{\mathrm{clip,THL}}-g_{\mathrm{clip}}\|
\leq
A_{\max}G\,\mathbb{E}_{\mu}|\tilde\rho-\rho|.
\]
By \Cref{cor:thl-ratio}, $|\tilde\rho-\rho|\leq(e^{\bar\delta_{\mathrm{THL}}}-1)\rho$, and $\mathbb{E}_{\mu}\rho=1$.  This proves the clipped THL bias bound.  The same envelope gives $\tilde\rho^2\leq e^{2\bar\delta_{\mathrm{THL}}}\rho^2$, hence
\[
\mathbb{E}_{\mu}\|\tilde\rho h\|^2
\leq
e^{2\bar\delta_{\mathrm{THL}}}A_{\max}^2G^2\mathbb{E}_{\mu}\rho^2
=
e^{2\bar\delta_{\mathrm{THL}}}A_{\max}^2G^2(1+\chi^2(\pi\|\mu)).
\]
\end{proof}

This theorem is the formal divergence-sensitivity statement for \MOneShort{}: exact correction removes sampling bias only under coverage, while variance scales with the density-ratio second moment.  Clipping controls variance but introduces tail bias, and \THL{} residuals multiply the ratio.  This is the same reason off-policy policy-gradient methods require explicit correction or truncation in distributed actor-learner settings~\citep{degris2012offpolicy,Espeholt2018IMPALA}.

\begin{proposition}[Naive \MOneShort{} can be anti-aligned under GRPO advantages]
\label{prop:prp-anti-align}
Assume a one-prompt, three-action contextual bandit with actions $a_1,a_2,a_3$ and binary rewards
\[
r(a_1)=1,\qquad r(a_2)=r(a_3)=0.
\]
Let $0<\eta\leq 1/25$ and let the learner and peer distributions be
\[
\pi^{(n)}_\theta(a_1)=\frac13,\qquad
\pi^{(n)}_\theta(a_2)=\frac23-\eta,\qquad
\pi^{(n)}_\theta(a_3)=\eta,
\]
\[
\mu(a_1)=\eta,\qquad
\mu(a_2)=\eta,\qquad
\mu(a_3)=1-2\eta.
\]
Use the population GRPO baseline $b=\mathbb{E}_{a\sim\pi^{(n)}_\theta}r(a)=1/3$, so
\[
A(a_1)=\frac23,\qquad A(a_2)=A(a_3)=-\frac13.
\]
With a tabular softmax parameterization over logits, the naive peer-pooling gradient
\[
g_{\mathrm{naive}}=\mathbb{E}_{a\sim\mu}[A(a)\nabla_\theta\log\pi^{(n)}_\theta(a)]
\]
has negative inner product with the on-policy gradient
\[
g_{\mathrm{on}}=\mathbb{E}_{a\sim\pi^{(n)}_\theta}[A(a)\nabla_\theta\log\pi^{(n)}_\theta(a)].
\]
In fact,
\[
\langle g_{\mathrm{on}},g_{\mathrm{naive}}\rangle
=
\frac{2}{3}\eta^3-\eta^2+\frac{5}{9}\eta-\frac{2}{81}
<0.
\]
Moreover $\chi^2(\pi^{(n)}_\theta\|\mu)\to\infty$ as $\eta\to0$.
\end{proposition}

\begin{proof}
For a tabular softmax, $\nabla_\theta\log\pi(a_i)=e_i-\pi$, where $\pi$ is the vector of action probabilities.  Since the population GRPO advantage has zero mean under the learner, $\mathbb{E}_{\pi}A=0$, the on-policy gradient in logit coordinates is
\[
g_{\mathrm{on},i}
=
\pi(a_i)A(a_i),
\]
hence
\[
g_{\mathrm{on}}
=
\left(
\frac29,\;
-\frac13\left(\frac23-\eta\right),\;
-\frac{\eta}{3}
\right).
\]
For the naive peer estimator, the expectation is taken under $\mu$ but the score is still the learner score:
\[
g_{\mathrm{naive}}
=
\mathbb{E}_{\mu}[A(a)e_a]
-
\pi\,\mathbb{E}_{\mu}A(a).
\]
Here
\[
\mathbb{E}_{\mu}A
=
\eta\frac23+\eta\left(-\frac13\right)+(1-2\eta)\left(-\frac13\right)
=
\eta-\frac13.
\]
Therefore
\[
g_{\mathrm{naive}}
=
\left(
\frac{2\eta}{3},-\frac{\eta}{3},-\frac{1-2\eta}{3}
\right)
-
\left(
\frac13,\frac23-\eta,\eta
\right)\left(\eta-\frac13\right).
\]
Taking the dot product with $g_{\mathrm{on}}$ and simplifying gives
\[
\langle g_{\mathrm{on}},g_{\mathrm{naive}}\rangle
=
\frac{2}{3}\eta^3-\eta^2+\frac{5}{9}\eta-\frac{2}{81}.
\]
For $0<\eta\leq1/25$,
\[
\frac{2}{3}\eta^3-\eta^2+\frac{5}{9}\eta-\frac{2}{81}
\leq
\frac{2}{3\cdot25^3}+\frac{1}{45}-\frac{2}{81}
<0.
\]
Finally,
\[
\chi^2(\pi\|\mu)
=
\frac{(1/3)^2}{\eta}
+
\frac{(2/3-\eta)^2}{\eta}
+
\frac{\eta^2}{1-2\eta}
-1,
\]
which diverges like $1/\eta$ as $\eta\to0$.  The example shows anti-alignment caused by the wrong behavior distribution, not a universal failure of corrected \MOneShort{}.
\end{proof}

\subsection{Value-level sharing (\MThreeShort{}): variance and direct-support preservation}
\label{subsec:theory-xgrpo}

\begin{assumption}[Stop-gradient prompt baseline]
\label{ass:xgrpo-baseline}
Fix prompt $x$ and learner $n$.  The policy is differentiable, sums and gradients can be interchanged, and any baseline $b(x)$ used in the advantage is treated as a stop-gradient scalar independent of the sampled response $y\sim\pi^{(n)}_\theta(\cdot\mid x)$.
\end{assumption}

\begin{lemma}[Prompt-level pooled baselines are unbiased]
\label{lem:xgrpo-baseline}
Under Assumption~\ref{ass:xgrpo-baseline},
\[
\mathbb{E}_{y\sim\pi^{(n)}_\theta(\cdot\mid x)}
\left[
b(x)\nabla_\theta\log\pi^{(n)}_\theta(y\mid x)
\right]
=0.
\]
Consequently, replacing $r(x,y)$ by $r(x,y)-b(x)$ preserves the expected score-gradient direction.
\end{lemma}

\begin{proof}
Because $b(x)$ does not depend on the sampled response,
\[
\mathbb{E}_{\pi^{(n)}_\theta}
[b(x)\nabla_\theta\log\pi^{(n)}_\theta(y\mid x)]
=
b(x)\sum_y \pi^{(n)}_\theta(y\mid x)
\nabla_\theta\log\pi^{(n)}_\theta(y\mid x).
\]
Using $\pi\nabla\log\pi=\nabla\pi$ gives
\[
b(x)\sum_y \nabla_\theta\pi^{(n)}_\theta(y\mid x)
=
b(x)\nabla_\theta\sum_y\pi^{(n)}_\theta(y\mid x)
=
b(x)\nabla_\theta 1
=
0.
\]
Thus pooled scalar reward statistics can change variance through the advantage value but not the expected score-gradient direction through a response-dependent term.
\end{proof}

\begin{assumption}[Baseline comparison for \MThreeShort{}]
\label{ass:xgrpo-variance}
For each learner-scored rollout, baselines are stop-gradient and independent of that rollout; leave-one-out estimates are one way to satisfy this condition.  Define
\[
s_\theta(y)=\nabla_\theta\log\pi^{(n)}_\theta(y\mid x),
\qquad
H_n(x)=\mathbb{E}_{y\sim\pi^{(n)}_\theta}\|s_\theta(y)\|^2,
\]
with $0<H_n(x)<\infty$.  Let $b_n(x)$ be the learner-only baseline and $b_{\mathrm{pool}}(x)$ the \MThreeShort{} pooled baseline.  Define the variance-minimizing scalar baseline
\[
b^*(x)
=
\frac{
\mathbb{E}_{y\sim\pi^{(n)}_\theta}
[r(x,y)\|s_\theta(y)\|^2]
}{
\mathbb{E}_{y\sim\pi^{(n)}_\theta}
\|s_\theta(y)\|^2
},
\]
the local baseline error
\[
D_n(x)=b_n(x)-b^*(x),
\]
and the pooled correction
\[
\Delta_n(x)=b_n(x)-b_{\mathrm{pool}}(x).
\]
\end{assumption}

\begin{theorem}[When pooled \MThreeShort{} normalization reduces variance]
\label{thm:xgrpo-variance}
Under Assumption~\ref{ass:xgrpo-variance}, the conditional variance difference between the pooled-baseline and learner-baseline estimators is
\[
\operatorname{Var}((r-b_{\mathrm{pool}})s_\theta\mid x)
-
\operatorname{Var}((r-b_n)s_\theta\mid x)
=
H_n(x)\left(\Delta_n(x)^2-2D_n(x)\Delta_n(x)\right).
\]
Averaged over prompts, pooled normalization reduces variance if and only if
\[
\mathbb{E}_{x}
\left[H_n(x)D_n(x)\Delta_n(x)\right]
\geq
\frac12
\mathbb{E}_{x}
\left[H_n(x)\Delta_n(x)^2\right],
\]
and it increases variance when the reverse strict inequality holds.
\end{theorem}

\begin{proof}
For a deterministic baseline $b$, the conditional second moment is
\[
\mathbb{E}[\|(r-b)s_\theta\|^2\mid x]
=
\mathbb{E}[r^2\|s_\theta\|^2\mid x]
-
2b\,\mathbb{E}[r\|s_\theta\|^2\mid x]
+
b^2H_n(x).
\]
The expected gradient does not depend on $b$ by \Cref{lem:xgrpo-baseline}; hence minimizing the variance is equivalent to minimizing this quadratic second moment.  Differentiating with respect to $b$ yields
\[
-2\mathbb{E}[r\|s_\theta\|^2\mid x]+2bH_n(x)=0,
\]
so the minimizer is $b^*(x)$.  Completing the square gives
\[
\operatorname{Var}((r-b)s_\theta\mid x)
=
\operatorname{Var}((r-b^*)s_\theta\mid x)
+
H_n(x)(b-b^*(x))^2.
\]
Since $b_{\mathrm{pool}}=b_n-\Delta_n$ and $D_n=b_n-b^*$,
\[
b_{\mathrm{pool}}-b^*
=
D_n-\Delta_n.
\]
Therefore
\[
\begin{aligned}
&\operatorname{Var}((r-b_{\mathrm{pool}})s_\theta\mid x)
-
\operatorname{Var}((r-b_n)s_\theta\mid x)\\
&\qquad =
H_n(x)\left((D_n-\Delta_n)^2-D_n^2\right)
=
H_n(x)\left(\Delta_n^2-2D_n\Delta_n\right).
\end{aligned}
\]
Averaging over $x$ gives the stated condition.  Thus peer statistics help exactly when the pooled correction is sufficiently aligned with the local error $D_n$; otherwise the quadratic correction term dominates and variance can increase.
\end{proof}

\begin{theorem}[No direct peer-success score term in \MThreeShort{}]
\label{thm:xgrpo-direct-support}
Assume the \MThreeShort{} update for learner $n$ uses only learner samples $Y_n(x)$ in the actor score term and uses peer data only through stop-gradient scalar reward statistics.  Conditional on the realized learner group $Y_n(x)$ and the scalar reward multiset $\{r(x,y_{m,j})\}_{m,j}$, the \MThreeShort{} gradient has the form
\[
g_{\MThreeShort{}}(x)
=
\frac{1}{K}\sum_{i=1}^K
\alpha_i(x)\,
\nabla_\theta\log\pi^{(n)}_\theta(y_{n,i}\mid x),
\]
where each $\alpha_i(x)$ is a scalar depending on rewards and pooled normalization.  Therefore, for any peer-only success
\[
y^*\in
\mathcal{S}_x\cap\bigcup_{m\neq n}Y_m(x)
\quad\text{with}\quad
y^*\notin Y_n(x),
\]
the update contains no direct score-function term
$\nabla_\theta\log\pi^{(n)}_\theta(y^*\mid x)$.  In particular, two peer pools with the same scalar reward multiset but different successful texts induce the same \MThreeShort{} actor gradient.
\end{theorem}

\begin{proof}
By construction, \MThreeShort{} replaces the learner's local baseline and scale by pooled scalar reward statistics, but the actor still evaluates log-probabilities only on $y_{n,i}\sim\pi^{(n)}_\theta(\cdot\mid x)$.  For a realized batch, the advantage assigned to $y_{n,i}$ is a scalar
\[
\alpha_i(x)
=
\frac{r(x,y_{n,i})-\mu_{\mathrm{pool}}(x)}
{\sigma_{\mathrm{pool}}(x)+\epsilon}
\]
or the corresponding unnormalized pooled-baseline value.  Since the pooled statistics are stop-gradient scalars, differentiation gives
\[
\nabla_\theta
\left(
-\frac1K\sum_{i=1}^K \alpha_i(x)
\log\pi^{(n)}_\theta(y_{n,i}\mid x)
\right)
=
-\frac1K\sum_{i=1}^K
\alpha_i(x)\nabla_\theta\log\pi^{(n)}_\theta(y_{n,i}\mid x).
\]
After changing sign to view this as an ascent direction, the displayed form follows.  If $y^*\notin Y_n(x)$, no summand is indexed by $y^*$.  The identity of a peer success affects \MThreeShort{} only through its scalar reward, so replacing $y^*$ by a different successful peer text with the same reward leaves all $\alpha_i(x)$ unchanged and leaves the actor gradient unchanged.  The theorem does not assert that neural shared parameters or a softmax normalizer have zero indirect effect on unsampled strings; it asserts that \MThreeShort{} has no direct trajectory-conditioned score term for the peer-only success.
\end{proof}

\begin{corollary}[Hard-support boundary]
\label{cor:xgrpo-hard-support}
Under the assumptions of \Cref{thm:xgrpo-direct-support}, if a policy class has a hard support constraint with $\pi^{(n)}_\theta(y^*\mid x)=0$ in a neighbourhood of $\theta$ and $\nabla_\theta\pi^{(n)}_\theta(y^*\mid x)=0$, then \MThreeShort{} cannot create direct probability mass on $y^*$ from scalar peer rewards alone.  The probability that learner $n$ observes at least one success in its own rollout group remains
\[
\Pr[Y_n(x)\cap\mathcal{S}_x\neq\varnothing]
=
1-(1-p_n(x))^K,
\]
independent of pooled reward statistics.
\end{corollary}

\begin{proof}
The first claim follows from \Cref{thm:xgrpo-direct-support}: no term involving $\log\pi^{(n)}_\theta(y^*\mid x)$ is introduced, and the hard-support assumption gives zero derivative of the probability itself.  The rollout probability is a direct binomial complement.  Each of the $K$ learner samples succeeds with probability $p_n(x)$ and fails with probability $1-p_n(x)$; independence gives
\[
\Pr[Y_n(x)\cap\mathcal{S}_x=\varnothing]=(1-p_n(x))^K.
\]
Taking the complement gives the displayed expression.  Since \MThreeShort{} changes only the scalar normalization after samples are drawn, the sampling probability is unchanged.
\end{proof}

\begin{proposition}[Zero-support success boundary for \MThreeShort{}]
\label{prop:xgrpo-boundary}
Consider one prompt and two actions with
\[
\pi^{(n)}_\theta(a_1\mid x)=1,\qquad
\pi^{(n)}_\theta(a_2\mid x)=0,
\qquad
r(x,a_1)=0,\qquad
r(x,a_2)=1,
\]
and suppose a peer satisfies $\mu^{(m)}(a_2\mid x)>0$.  Under a hard-support parameterization satisfying \Cref{cor:xgrpo-hard-support}, \MThreeShort{} can change the scalar advantage assigned to $a_1$, but it produces no direct gradient term toward $a_2$.
\end{proposition}

\begin{proof}
Learner $n$ samples only $a_1$, so the \MThreeShort{} actor gradient is of the form
\[
g_{\MThreeShort{}}(x)
=
\alpha_1(x)\nabla_\theta\log\pi^{(n)}_\theta(a_1\mid x).
\]
If the peer reward pool contains $a_2$, then $\mu_{\mathrm{pool}}(x)>0$, so the scalar $\alpha_1(x)$ may become negative; for example, with pooled z-normalization,
\[
\alpha_1(x)=\frac{0-\mu_{\mathrm{pool}}(x)}{\sigma_{\mathrm{pool}}(x)+\epsilon}<0
\]
whenever the pool contains at least one success and one failure.  This reweights the learner-sampled failure.  However, no summand of the form
\[
\alpha_2(x)\nabla_\theta\log\pi^{(n)}_\theta(a_2\mid x)
\]
appears, because $a_2\notin Y_n(x)$.  Under the hard-support condition, $\nabla_\theta\pi^{(n)}_\theta(a_2\mid x)=0$, so the peer-only success is not directly introduced by \MThreeShort{}.
\end{proof}

\subsection{Outcome-level sharing (\MFourShort{}): rescue-set gradients}
\label{subsec:theory-sgt}

\begin{definition}[Rescue gate probability]
\label{def:sgt-gate}
Assume independent $K$ rollouts per policy on prompt $x$.  The \MFourShort{} gate for learner $n$ fires when $Y_n(x)\cap\mathcal{S}_x=\varnothing$ and $\bigcup_{m\neq n}Y_m(x)\cap\mathcal{S}_x\neq\varnothing$.  Its probability is
\[
G_n(x)
=
(1-p_n(x))^K
\left[
1-\prod_{m\neq n}(1-p_m(x))^K
\right].
\]
\end{definition}

The identity follows directly: the learner failure event has probability $(1-p_n(x))^K$, while the probability of at least one peer success is the complement of all peers failing, namely $1-\prod_{m\neq n}(1-p_m(x))^K$.

\begin{lemma}[Self-limiting rescue activation]
\label{lem:sgt-self-limit}
Assume peer success probabilities $\{p_m(x):m\neq n\}$ are fixed when differentiating with respect to $p_n(x)$.  Then
\[
\frac{\partial G_n(x)}{\partial p_n(x)}
=
-K(1-p_n(x))^{K-1}
\left[
1-\prod_{m\neq n}(1-p_m(x))^K
\right]
\leq0.
\]
Thus the gate de-activates as the learner becomes more likely to solve the prompt.
\end{lemma}

\begin{proof}
Let
\[
Q_{-n}(x)=1-\prod_{m\neq n}(1-p_m(x))^K.
\]
Under the assumption, $Q_{-n}(x)$ is constant with respect to $p_n(x)$ and lies in $[0,1]$.  By \Cref{def:sgt-gate},
\[
G_n(x)=(1-p_n(x))^K Q_{-n}(x).
\]
Differentiating gives
\[
\frac{\partial G_n(x)}{\partial p_n(x)}
=
Q_{-n}(x)\frac{\partial}{\partial p_n(x)}(1-p_n(x))^K
=
-K(1-p_n(x))^{K-1}Q_{-n}(x).
\]
All factors outside the leading minus sign are nonnegative, proving the inequality.  The activation probability is therefore structurally sparse: it is high only where the learner fails and peers provide a success.
\end{proof}

\begin{definition}[Peer-success conditional]
\label{def:peer-success-conditional}
If $\sum_{m\neq n}p_m(x)>0$, define the population peer-success conditional by
\[
\tau_{-n}(y\mid x)
=
\frac{
\sum_{m\neq n}\pi^{(m)}_\theta(y\mid x)\mathbf{1}\{y\in\mathcal{S}_x\}
}{
\sum_{m\neq n}p_m(x)
}.
\]
Equivalently, choose a peer uniformly from $\{m:m\neq n\}$, draw from that peer, and condition on verifier success; the common factor $1/(M-1)$ cancels.
\end{definition}

\begin{assumption}[Scoreable rescue success]
\label{ass:sgt-scoreable}
On a rescue event for prompt $x$, the selected peer success $y^*$ is scoreable by learner $n$ after \THL{} retokenization:
\[
\pi^{(n)}_\theta(y^*\mid x)>0,
\qquad
\nabla_\theta\log\pi^{(n)}_\theta(y^*\mid x)\neq0.
\]
For the population statement, $\tau_{-n}(\cdot\mid x)$ is fixed during the infinitesimal learner update, and
$\theta\mapsto \KL(\tau_{-n}(\cdot\mid x)\,\|\,\pi^{(n)}_\theta(\cdot\mid x))$ is twice continuously differentiable.
\end{assumption}

\begin{theorem}[\MFourShort{} supplies a rescue-set score direction]
\label{thm:sgt-rescue-gradient}
Assume \Cref{ass:sgt-scoreable} and condition on a rescue event:
\[
Y_n(x)\cap\mathcal{S}_x=\varnothing,
\qquad
y^*\in\mathcal{S}_x\cap\bigcup_{m\neq n}Y_m(x).
\]
The \MThreeShort{} update on this same realized batch contains no direct score term for $y^*$ by \Cref{thm:xgrpo-direct-support}.  In contrast, a gradient step on the \MFourShort{} term
\[
-\lambda\log\pi^{(n)}_\theta(y^*\mid x)
\]
with step size $\eta$ updates $\theta^+=\theta+\eta\lambda\nabla_\theta\log\pi^{(n)}_\theta(y^*\mid x)$ and satisfies
\[
\log\pi^{(n)}_{\theta^+}(y^*\mid x)
=
\log\pi^{(n)}_{\theta}(y^*\mid x)
+
\eta\lambda
\left\|
\nabla_\theta\log\pi^{(n)}_\theta(y^*\mid x)
\right\|^2
+
O(\eta^2).
\]
Moreover, for the population rescue loss
\[
\mathcal{L}_{\mathrm{SGT}}^{(n)}(x)
=
-\mathbb{E}_{y\sim\tau_{-n}(\cdot\mid x)}
\log\pi^{(n)}_\theta(y\mid x),
\]
the same descent step decreases
\[
\KL(\tau_{-n}(\cdot\mid x)\,\|\,\pi^{(n)}_\theta(\cdot\mid x))
\]
to first order whenever its gradient is nonzero.
\end{theorem}

\begin{proof}
By the rescue condition, learner $n$ has no successful response in $Y_n(x)$.  \Cref{thm:xgrpo-direct-support} therefore implies that \MThreeShort{} can only reweight the learner-sampled failures in $Y_n(x)$; the peer-only success $y^*$ contributes no summand
$\nabla_\theta\log\pi^{(n)}_\theta(y^*\mid x)$.

For \MFourShort{}, the auxiliary loss on the selected success is
\[
\mathcal{L}_{\mathrm{SGT}}^{(n)}(x,y^*)
=
-\log\pi^{(n)}_\theta(y^*\mid x).
\]
A descent step on $\lambda\mathcal{L}_{\mathrm{SGT}}^{(n)}$ is
\[
\theta^+
=
\theta-\eta\lambda\nabla_\theta\mathcal{L}_{\mathrm{SGT}}^{(n)}
=
\theta+\eta\lambda\nabla_\theta\log\pi^{(n)}_\theta(y^*\mid x).
\]
Taylor expansion of $\log\pi^{(n)}_\theta(y^*\mid x)$ around $\theta$ gives
\[
\log\pi^{(n)}_{\theta^+}(y^*\mid x)
=
\log\pi^{(n)}_{\theta}(y^*\mid x)
+
\eta\lambda
\left\|
\nabla_\theta\log\pi^{(n)}_\theta(y^*\mid x)
\right\|^2
+
O(\eta^2),
\]
which is a strict first-order increase under Assumption~\ref{ass:sgt-scoreable}.  For the population statement,
\[
-\mathbb{E}_{\tau_{-n}}\log\pi^{(n)}_\theta(y\mid x)
=
H(\tau_{-n})
+
\KL(\tau_{-n}\,\|\,\pi^{(n)}_\theta(\cdot\mid x)).
\]
The entropy term is constant in $\theta$, so descending the cross-entropy is descending the KL.  If
$K(\theta)=\KL(\tau_{-n}\,\|\,\pi^{(n)}_\theta)$ and
$\theta^+=\theta-\eta\nabla_\theta K(\theta)$, then
\[
K(\theta^+)
=
K(\theta)
-
\eta\|\nabla_\theta K(\theta)\|^2
+
O(\eta^2).
\]
For sufficiently small $\eta$, the first-order term is strictly negative whenever $\nabla_\theta K(\theta)\neq0$.
\end{proof}

The theorem formalizes the outcome-level channel: \MFourShort{} adds a positive score direction only on the rescue subset, with expected contribution weighted by the gate probability $G_n(x)$. \Cref{lem:sgt-self-limit} shows that this weight shrinks as the learner becomes capable on the prompt. The gate therefore distinguishes \MFourShort{} from a global distillation objective: reward-ranked fine-tuning and rejection-sampling fine-tuning train on filtered successful samples without the same concurrent failure-triggered gate~\citep{dong2023raft,xiong2025minimalist}.

\subsection{Gated perturbation of the base GRPO update}
\label{subsec:theory-stability}

\begin{assumption}[Bounded gated auxiliary gradient]
\label{ass:sgt-perturb}
For learner $n$, write the base update direction as $g_{\mathrm{GRPO}}(x)$ and the gated auxiliary direction as
\[
g_{\mathrm{SGT}}(x)=I_{\mathrm{gate}}(x)\,\bar g_{\mathrm{SGT}}(x),
\]
where $I_{\mathrm{gate}}(x)$ is the \MFourShort{} gate indicator, $\mathbb{E}[I_{\mathrm{gate}}(x)\mid x]=G_n(x)$, and
$\|\bar g_{\mathrm{SGT}}(x)\|\leq G_S$ whenever the gate fires.  The combined update is
\[
\theta_{\mathrm{comb}}
=
\theta+\eta\left(g_{\mathrm{GRPO}}+\lambda g_{\mathrm{SGT}}\right),
\]
while the base update is
\[
\theta_{\mathrm{base}}=\theta+\eta g_{\mathrm{GRPO}}.
\]
\end{assumption}

\begin{proposition}[Expected auxiliary perturbation is controlled by rescue frequency]
\label{prop:sgt-perturb}
Under Assumption~\ref{ass:sgt-perturb},
\[
\mathbb{E}\|\theta_{\mathrm{comb}}-\theta_{\mathrm{base}}\|
\leq
\eta\lambda G_S\,\mathbb{E}_{x\sim\mathcal{D}}[G_n(x)],
\]
and
\[
\mathbb{E}\|\theta_{\mathrm{comb}}-\theta_{\mathrm{base}}\|^2
\leq
\eta^2\lambda^2 G_S^2\,\mathbb{E}_{x\sim\mathcal{D}}[G_n(x)].
\]
If, additionally, the local policy KL between the two post-update policies is bounded along the segment by
\[
\KL(\pi_{\theta_{\mathrm{base}}}\,\|\,\pi_{\theta_{\mathrm{comb}}})
\leq
\frac{\kappa}{2}
\|\theta_{\mathrm{comb}}-\theta_{\mathrm{base}}\|^2
+
O(\|\theta_{\mathrm{comb}}-\theta_{\mathrm{base}}\|^3),
\]
then
\[
\mathbb{E}\,
\KL(\pi_{\theta_{\mathrm{base}}}\,\|\,\pi_{\theta_{\mathrm{comb}}})
\leq
\frac{\kappa}{2}\eta^2\lambda^2G_S^2
\mathbb{E}_{x\sim\mathcal{D}}[G_n(x)]
+
O(\eta^3\lambda^3G_S^3\mathbb{E}_{x\sim\mathcal{D}}[G_n(x)]).
\]
\end{proposition}

\begin{proof}
The difference between the combined and base updates is exactly
\[
\theta_{\mathrm{comb}}-\theta_{\mathrm{base}}
=
\eta\lambda g_{\mathrm{SGT}}
=
\eta\lambda I_{\mathrm{gate}}\bar g_{\mathrm{SGT}}.
\]
Thus, conditional on prompt $x$,
\[
\mathbb{E}\!\left[\|\theta_{\mathrm{comb}}-\theta_{\mathrm{base}}\|\mid x\right]
\leq
\eta\lambda G_S\,\mathbb{E}[I_{\mathrm{gate}}\mid x]
=
\eta\lambda G_S G_n(x).
\]
Averaging over $x\sim\mathcal{D}$ proves the first display.  Similarly,
\[
\mathbb{E}\!\left[\|\theta_{\mathrm{comb}}-\theta_{\mathrm{base}}\|^2\mid x\right]
\leq
\eta^2\lambda^2G_S^2\,\mathbb{E}[I_{\mathrm{gate}}\mid x]
=
\eta^2\lambda^2G_S^2G_n(x),
\]
which gives the second display.  The KL statement follows by substituting this squared-norm bound into the assumed local expansion.  Since the auxiliary displacement is zero whenever the gate is off, the cubic term is also weighted by the gate frequency, giving the stated remainder.
\end{proof}

This is the stability distinction relevant to the three-channel comparison: \MFourShort{} leaves the on-policy GRPO batch intact and adds a sparse, gate-weighted auxiliary direction, whereas \MOneShort{} changes the sampled trajectory distribution and inherits the density-ratio variance term in \Cref{thm:prp-bv}.

\subsection{Unified comparison theorem}
\label{subsec:theory-unified}

\begin{assumption}[Shared conditions]
\label{ass:unified}
Assume the contextual-bandit setup of \Cref{subsec:theory-setup}.  Assume the \THL{} span and clipping convention in \Cref{ass:thl-span}, the \MOneShort{} coverage and boundedness assumptions in \Cref{ass:prp}, the \MThreeShort{} stop-gradient baseline assumptions in \Cref{ass:xgrpo-baseline,ass:xgrpo-variance}, and the \MFourShort{} scoreability and bounded-gradient assumptions in \Cref{ass:sgt-scoreable,ass:sgt-perturb}.  All rewards are produced by the same verifier $r(x,y)$.
\end{assumption}

\begin{theorem}[Channel comparison: coupling, variance, support, and rescue signal]
\label{thm:unified}
Under Assumption~\ref{ass:unified}, the three mutual-RL channels differ structurally as follows.
\begin{enumerate}[leftmargin=*]
    \item \textbf{\MOneShort{} transfers trajectories and is divergence-sensitive.}  Corrected data-level sharing is unbiased only under coverage and has second moment controlled by
    \[
    1+\chi^2(\pi^{(n)}_\theta(\cdot\mid x)\,\|\,\mu^{(m)}(\cdot\mid x)).
    \]
    Clipping trades this variance for explicit tail bias, while \THL{} residuals multiply the ratio by the envelope in \Cref{cor:thl-ratio}.
    \item \textbf{\MThreeShort{} preserves on-policy actor sampling.}  Pooled scalar rewards are unbiased baselines by \Cref{lem:xgrpo-baseline} and reduce variance exactly under the condition in \Cref{thm:xgrpo-variance}.  However, by \Cref{thm:xgrpo-direct-support}, \MThreeShort{} has no direct score-function term for a peer-only successful trajectory.
    \item \textbf{\MFourShort{} is a rescue-set update.}  It activates with probability $G_n(x)$ in \Cref{def:sgt-gate}, self-limits by \Cref{lem:sgt-self-limit}, and on a rescue event supplies the positive score direction in \Cref{thm:sgt-rescue-gradient}.  Its expected perturbation relative to base GRPO is controlled by $\lambda G_S\mathbb{E}_xG_n(x)$ through \Cref{prop:sgt-perturb}.
\end{enumerate}
\end{theorem}

\begin{proof}
For \MOneShort{}, \Cref{thm:prp-bv} proves the exact-ratio unbiasedness and the variance bound
$A_{\max}^2G^2(1+\chi^2(\pi\|\mu))$.  The same theorem shows how PPO-style clipping bounds the second moment by $(1+\epsilon)^2A_{\max}^2G^2$ but introduces the tail bias
\[
\mathbb{E}_{\mu}\left[
(\rho-(1+\epsilon))_+
+
((1-\epsilon)-\rho)_+
\right].
\]
\Cref{thm:thl-residual} derives the denominator residual from the THL alignment procedure, and \Cref{cor:thl-ratio} converts this residual into the multiplicative envelope
$e^{-\delta_{\mathrm{THL}}}\rho\leq\tilde\rho\leq e^{\delta_{\mathrm{THL}}}\rho$.  \Cref{prop:prp-anti-align} shows that ignoring the behavior distribution can even make the naive advantage-weighted direction anti-aligned with the on-policy direction.

For \MThreeShort{}, \Cref{lem:xgrpo-baseline} establishes that prompt-level scalar baselines do not bias the expected score-gradient direction.  \Cref{thm:xgrpo-variance} then gives the exact variance difference
\[
\mathbb{E}_x\left[H_n(x)(\Delta_n(x)^2-2D_n(x)\Delta_n(x))\right],
\]
so pooled statistics help precisely when they move the learner baseline toward the reward-weighted optimum strongly enough to dominate their own quadratic error.  \Cref{thm:xgrpo-direct-support} proves the support boundary relevant to mutual RL: because peer information enters only through scalars, the actor update contains no trajectory-conditioned score term for a peer-only success.  \Cref{cor:xgrpo-hard-support} and \Cref{prop:xgrpo-boundary} record the corresponding hard-support boundary.

For \MFourShort{}, \Cref{def:sgt-gate} computes the rescue probability from the two events that define the mechanism: learner failure and peer success.  \Cref{lem:sgt-self-limit} differentiates this probability and shows it decreases with $p_n(x)$.  On the same rescue event where \MThreeShort{} only reweights learner failures, \Cref{thm:sgt-rescue-gradient} proves that the SGT term increases the log-probability of the selected verified peer success to first order and decreases the peer-success KL in expectation.  Finally, \Cref{prop:sgt-perturb} shows that this auxiliary direction is weighted by the rescue frequency, giving an expected perturbation controlled by $\lambda G_S\mathbb{E}_xG_n(x)$.  These statements establish the claimed ordering by coupling type, variance exposure, direct support behavior, and dependence on \THL{} residuals.
\end{proof}

\subsection{Scope of the theoretical analysis}
\label{subsec:theory-limitations}

The theory targets the prompt-level, critic-free RLVR regime used in the experiments: a policy samples a complete response, a verifier supplies a scalar outcome, and GRPO-style updates use response-level score terms. Within this setting, the results establish bias, variance, support, THL-residual, and rescue-set perturbation statements under the bounded-gradient, bounded-advantage, clipping, and local smoothness assumptions stated in the theorem blocks. The channel comparison is therefore a structural analysis of the update directions used by \MOneShort{}, \MThreeShort{}, and \MFourShort{}. Verifier correctness enters through \(r(x,y)\): the theory characterizes transfer of the verifier-defined success signal. The exact THL conservation result uses the WordSpans refinement convention; for non-whitespace scripts, the character-span extension evaluated in App.~\ref{app:thl-diagnostics} supplies the corresponding alignment primitive.

\section{Prompt Templates for Evaluation}
\label{app:prompt-templates}

All auxiliary QA benchmarks use a common multiple-choice interface with an explicit instruction to reason step by step and return the final answer inside \(\boxed{\cdot}\). The templates below are the exact evaluation prompts; braces denote dataset fields.

\paragraph{HellaSwag.} For HellaSwag, we present the story followed by four options:
\begin{quote}
\texttt{\{story\}}\\[0.5ex]
\texttt{Which option (A-D) best completes the sentence?}\\
\texttt{Let's think step by step and output the final answer within \textbackslash boxed\{\}.}\\[0.5ex]
\texttt{\{joined\}}  \hfill \textit{\small (A–D candidate endings)}
\end{quote}

\paragraph{AI2-ARC.} For ARC, we show the question and four choices:
\begin{quote}
\texttt{\{question\}}\\[0.5ex]
\texttt{Which option (A-D) is the correct answer?}\\
\texttt{Let's think step by step and output the final answer within \textbackslash boxed\{\}.}\\[0.5ex]
\texttt{\{choices\_text\}} \hfill \textit{\small (A–D answer options)}
\end{quote}

\paragraph{BoolQ.} For BoolQ, we frame the task as a yes/no question grounded in a passage:
\begin{quote}
\texttt{Passage: \{passage\}}\\[0.5ex]
\texttt{Question: \{question\}}\\[0.5ex]
\texttt{Answer the question with Yes or No based on the passage above.}\\
\texttt{Let's think step by step and output the final answer within \textbackslash boxed\{\}.}
\end{quote}

\paragraph{OpenBookQA.} For OpenBookQA, we include the supporting fact and question, then four options:
\begin{quote}
\texttt{Fact: \{fact1\}}\\[0.5ex]
\texttt{Question: \{question\_stem\}}\\[0.5ex]
\texttt{Which option (A-D) is the correct answer?}\\
\texttt{Let's think step by step and output the final answer within \textbackslash boxed\{\}.}\\[0.5ex]
\texttt{\{choices\_text\}} \hfill \textit{\small (A–D answer options)}
\end{quote}

\paragraph{PIQA.} For PIQA, we present the goal and two candidate solutions:
\begin{quote}
\texttt{\{goal\}}\\[0.5ex]
\texttt{Which option (A-B) is the better solution?}\\
\texttt{Let's think step by step and output the final answer within \textbackslash boxed\{\}.}\\[0.5ex]
\texttt{A. \{sol1\}}\\
\texttt{B. \{sol2\}}
\end{quote}

\paragraph{Social IQa.} For Social IQa, we show the context, question, and three options:
\begin{quote}
\texttt{Context: \{context\}}\\[0.5ex]
\texttt{Question: \{question\}}\\[0.5ex]
\texttt{Which option (A-C) is the correct answer?}\\
\texttt{Let's think step by step and output the final answer within \textbackslash boxed\{\}.}\\[0.5ex]
\texttt{A. \{answerA\}}\\
\texttt{B. \{answerB\}}\\
\texttt{C. \{answerC\}}
\end{quote}

\section{Analogies to Classical Closed-Loop System Identification}
\label{app:sysid-analogy}

This appendix uses classical \emph{closed-loop} system identification (SysID)~\cite{ljung1998system} as an interpretive lens for the three communication channels. The analogy is useful because closed-loop estimators fail for the same structural reasons that direct rollout sharing becomes fragile: correlated data, insufficient excitation, and non-stationary feedback.

Closed-loop SysID separates three ingredients that also organize Mutual RL. First, effective regressors should be exogenous enough that regressor-noise correlations do not bias the estimate. Second, the input should be informative, in the sense of persistent excitation, so the data covers modes that matter. Third, the feedback loop should not couple the data distribution so tightly to the current estimator that variance-reduction intuitions break. \MOneShort{} stresses all three conditions by reusing peer trajectories inside a learner update. \MThreeShort{} preserves the stable regime by sharing only scalar rewards, but it inherits the information ceiling of unchanged exploration support. \MFourShort{} adds a gate that routes high-signal peer successes only on learner-failure prompts, reducing coupling while adding trajectory support on the rescue subset.

\paragraph{Closed-loop SysID in one equation.}
A canonical closed-loop SysID setting considers an unknown plant $\mathcal{G}$ controlled by a known stabilizing controller $\mathcal{K}$.
For a scalar (or LTI MIMO) discrete-time model, one common abstraction is
\begin{equation}
\label{eq:sysid_closedloop}
\begin{aligned}
y_t &= \mathcal{G}(q^{-1})\,u_t + \mathcal{H}(q^{-1})\,e_t, \\
u_t &= \mathcal{K}(q^{-1})\,(r_t - y_t),
\end{aligned}
\end{equation}
where $e_t$ is an exogenous disturbance/noise process and $r_t$ is a known reference signal.
In open-loop identification, many estimators rely (explicitly or implicitly) on the condition
$\Cov(u_t, e_t)=0$ (or, more generally, that the regressor is independent of the innovation).
In closed loop, this condition typically fails:
substituting $y_t$ into the controller equation shows that $u_t$ contains filtered versions of $e_t$, hence $\Cov(u_t,e_t)\neq 0$.
This is the classical \emph{endogeneity bias} problem: using an estimator that assumes ``$u$ independent of noise'' on closed-loop data yields biased (and sometimes unstable) estimates of $\mathcal{G}$.

\paragraph{Where the analogy \emph{does} and \emph{does not} apply.}
At a high level, Mutual RL also produces data under feedback: rollouts are generated by policies that are themselves updated from past rollouts.
Moreover, our multi-policy setting introduces additional feedback paths through the shared exchange (\SEE{}):
a learner consumes peer experience whose distribution depends on peer parameters, which in turn evolve during training.
However, there is a crucial misanalogy:
in classical SysID, the plant $\mathcal{G}$ is a \emph{fixed} object one aims to estimate, whereas in Mutual RL there is no fixed unknown ``plant'' corresponding to the target of inference.
The object of interest is the policy itself (or an optimal policy), which is being \emph{constructed} via reward maximization, not identified as ground-truth dynamics.
Thus, SysID should be viewed as a \emph{structural analogy about data under feedback}, not as a literal generative model of our learning problem.

Still, three SysID concepts map cleanly onto the behaviors we observe across \MOneShort{}, \MThreeShort{}, and \MFourShort{}:
\emph{endogeneity}, \emph{informativity/persistent excitation}, and \emph{decorrelation via gating/delay}.

\subsection{Endogeneity bias as a lens on \MOneShort{} instability}

\paragraph{Off-policy gradients already look like closed-loop estimation.}
For a learner policy $\pi^{(n)}_\theta$, the on-policy policy-gradient structure (suppressing token averaging and PPO clipping for readability) has the form
\begin{equation}
\label{eq:onpolicy_pg}
\nabla_\theta J^{(n)}(\theta)
\;\propto\;
\E_{x}\,\E_{y\sim \pi^{(n)}_{\theta_{\mathrm{old}}}(\cdot\mid x)}
\big[\nabla_\theta \log \pi^{(n)}_\theta(y\mid x)\,A(x,y)\big],
\end{equation}
which is exactly what Eq.~\ref{eq:grpo} approximates with a trust-region style importance ratio $w_\theta$ and clipping.
When \MOneShort{} pools peer rollouts (Sec.~\ref{subsec:prp}), the learner is (partly) replacing the sampling distribution in Eq.~\eqref{eq:onpolicy_pg} by peer behavior distributions $\mu^{(m)}$.
If we ignore this change in the sampling distribution, we get a biased gradient estimator:
\begin{equation}
\label{eq:naive_offpolicy_bias}
\E_{y\sim \mu^{(m)}(\cdot\mid x)}
\big[\nabla_\theta \log \pi^{(n)}_\theta(y\mid x)\,A(x,y)\big]
\;\neq\;
\E_{y\sim \pi^{(n)}_{\theta_{\mathrm{old}}}(\cdot\mid x)}
\big[\nabla_\theta \log \pi^{(n)}_\theta(y\mid x)\,A(x,y)\big].
\end{equation}
This is the familiar ``off-policy policy gradient instability'' story.
The SysID analogy is that Eq.~\ref{eq:naive_offpolicy_bias} implicitly assumes a form of independence between the regressor (the sampled trajectory under $\mu^{(m)}$) and the error/noise terms that appear inside $A(x,y)$ and the clipped importance weights.
In closed-loop SysID, the corresponding mistaken assumption is $\Cov(u_t,e_t)=0$ when data were generated under feedback.

\paragraph{PRP's naive denominator is analogous to an open-loop estimator on closed-loop data.}
In \MOneShort{}, the most fragile variant is the one that treats peer rollouts as if they were generated by the learner's own behavior policy (the ``naive on-policy assumption'' in \Secref{sec:methods}):
\[
w^{(n)}_\theta(x,y)=\exp\big(\ell^{(n)}_\theta(x,y)-\ell^{(n)}_{\mathrm{old}}(x,y)\big),
\quad y\sim \mu^{(m)}(\cdot\mid x),\ m\neq n.
\]
This is directly analogous to plugging closed-loop data into an open-loop estimator: the denominator encodes the wrong data-generating mechanism.
PPO/GRPO clipping can prevent extreme ratio blow-up, but it cannot remove the fundamental endogeneity: the update is still driven by trajectories whose likelihood under the assumed behavior model is systematically mis-specified.

\paragraph{Why ``importance correction'' helps but does not fully solve \MOneShort{}.}
The importance-corrected variant of \MOneShort{} instead uses the peer behavior probability in the denominator (aligned via \THL{}),
\[
w^{(n)}_\theta(x,y)=\exp\big(\ell^{(n)}_\theta(x,y)-\ell^{(m)}_{\text{aligned}}(x,y)\big),
\quad y\sim \mu^{(m)}(\cdot\mid x),
\]
which corresponds to the standard RL remedy: correct for the sampling distribution shift.
In SysID terms, this is akin to moving from an open-loop estimator to a closed-loop-capable method (e.g., prediction-error methods or instrumental-variable style corrections) that accounts for feedback-generated correlations.

However, two ``closed-loop'' issues remain and line up with what we observe empirically for \MOneShort{} even with correction (Appendix~\ref{subsec:prp-ablation}):
\begin{itemize}[topsep=0pt,itemsep=2pt,parsep=0pt,leftmargin=*]
\item \textbf{Variance amplification (importance weights).} Even if unbiased in principle, importance weights can have heavy tails when $\mu^{(m)}$ and $\pi^{(n)}$ diverge; clipping stabilizes but introduces bias, just as regularization in SysID trades bias for variance.
\item \textbf{Non-stationary data generation.} In classical closed-loop SysID, identification becomes harder when the controller is adaptive: the closed-loop mapping changes over time, and the regressor-noise correlation structure drifts.
In Mutual RL, peers evolve concurrently and may even become statistically dependent through shared prompts and shared exchange dynamics; thus $\mu^{(m)}$ is not a fixed behavior distribution over training.
This non-stationarity breaks the ``large-sample averaging'' intuition that typically supports off-policy corrections.
\end{itemize}

Taken together, the SysID lens emphasizes that \MOneShort{} is not merely ``off-policy'': it is off-policy \emph{under feedback and non-stationarity}, which is precisely the regime where both SysID and RL estimators are known to be fragile without additional structure (e.g., strong trust regions, explicit decorrelation, or conservative data selection).

\subsection{Persistent excitation and the boundary of value-only sharing}

\paragraph{SysID analogy: identifiability requires informative inputs.}
A core SysID concept is that even with infinite data, one may fail to identify parameters if the input does not sufficiently excite the system.
For many linear parametrizations, a persistent excitation (PE) condition can be stated as a rank condition on the regressor covariance, e.g.,
\[
\frac{1}{N}\sum_{t=1}^N \phi_t \phi_t^\top \to R_\phi \succ 0,
\]
where $\phi_t$ is the regression vector.
In closed loop, stabilizing controllers can \emph{reduce} excitation (by suppressing variation in $u_t$), making $R_\phi$ singular and destroying identifiability.

\paragraph{Mutual RL analogue: sharing statistics reduces variance but cannot add information.}
\MThreeShort{} (XGRPO) shares only scalar reward statistics across policies and keeps generation strictly local.
Concretely, each policy $n$ still samples $y^{(n)}\sim \pi^{(n)}(\cdot\mid x)$, but replaces its per-prompt normalization by pooled moments (Sec.~\ref{sec:methods}),
\[
A^{(n)}_i = \frac{r^{(n)}_i-\mu_{\text{pool}}}{\sigma_{\text{pool}}+\epsilon}.
\]
This is an archetypal \emph{variance-reduction} mechanism: it changes the baseline/scale of the learning signal but does not change the support of the data distribution (no new trajectories, no new action sequences).
In SysID terms, this is analogous to estimating a closed-loop transfer statistic more accurately without injecting new excitation: one can reduce estimation noise, but if the experiment never visits certain modes, those modes remain unidentifiable.

This perspective aligns with our empirical observation that \MThreeShort{} closely tracks Standalone GRPO across larger and more heterogeneous pools (Appendix~\ref{subsec:xgrpo-extended}): pooled scalar normalization improves the \emph{conditioning} of the update without changing the support of the trajectory distribution.

\paragraph{Loss of informativity in the value-only channel.}
One can interpret the pool of policies as defining a mixture behavior distribution over trajectories for each prompt,
\[
\bar\mu(\cdot\mid x) \;\propto\; \sum_{m=1}^M \mu^{(m)}(\cdot\mid x).
\]
\MThreeShort{} shares only low-dimensional statistics of returns under $\bar\mu$ (mean/variance), not trajectories. On prompts where the pool's trajectory distribution does not contain a correct sample, the pooled statistics are nearly identical regardless of what the learner does locally. In SysID language, the shared signal is \emph{uninformative} about directions outside the current closed-loop operating regime, and informative interventions require routing trajectories rather than scalar statistics. \MFourShort{} provides exactly that route by transferring verified peer successes on the rescue subset.

\subsection{Gating as decorrelation and ``exogenous'' supervision}

\paragraph{SysID lens: gating creates a decorrelating ``reference'' signal.}
The additional interpretation is about \emph{where the supervision comes from} and \emph{when it is applied}.
SGT triggers only when the learner fails but at least one peer succeeds on the same prompt.
Thus the supervised target $y^\star$ is not sampled from the learner's own closed-loop behavior at that time step; it is imported from a different policy's closed-loop.
This acts like an external reference injection in Eq.~\ref{eq:sysid_closedloop}: in closed-loop SysID, if one has access to an exogenous reference $r_t$ that is independent of disturbances, identifiability can be restored despite feedback because $r_t$ provides an ``instrument'' that breaks the endogenous correlation path.
Here, the peer success plays a related role: it provides a high-signal training target that is not generated by the learner's own (possibly collapsed) distribution on that prompt.

\paragraph{Gating reduces feedback coupling.}
The \MFourShort{} gate counters two coupling effects observed under direct rollout sharing (Sec.~\ref{subsec:diagnostics}). First, it does not apply imitation when the learner already succeeds, so the learner's update is not concentrated further on its current modes. Second, it conditions on verified peer successes only, so peer trajectories enter the learner's update solely on the rescue subset and only as a sparse positive sample.

\subsection{Implications for the three mutual-RL channels}

The SysID analogy summarizes the empirical taxonomy through three channel properties:
\begin{itemize}[topsep=0pt,itemsep=2pt,parsep=0pt,leftmargin=*]
\item \textbf{Endogeneity-aware sharing.} A mechanism that injects peer trajectories into a policy-gradient update should model the data-generating distribution through importance correction and control the resulting variance through trust regions, conservative selection, or delayed updates.
\item \textbf{Informativity requires excitation.} Sharing only low-dimensional statistics stabilizes the update without changing the support of the trajectory distribution. Routing verified trajectory support introduces new information into the pool.
\item \textbf{Structural decorrelation helps.} A gate between peer behavior and learner updates reduces harmful feedback coupling. \MFourShort{} uses this structure by turning peer experience into sparse, failure-conditioned outcome transfer rather than dense off-policy gradients.
\end{itemize}

Overall, Mutual RL is closer in spirit to \emph{adaptive control} than to pure SysID, because the policy changes as it learns. The closed-loop SysID lens still gives a compact summary of the observed regimes: aggressive trajectory sharing (\MOneShort{}) is where endogeneity and variance dominate; statistic sharing (\MThreeShort{}) is stable but does not expand the learner's sampled support; gated outcome transfer (\MFourShort{}) injects informative, low-variance supervision while avoiding the strongest closed-loop correlation paths.

\end{document}